\documentclass[12pt]{article}

\usepackage{xspace}
\newcommand{\eat}[1]{}
\newcommand{\etal}{\text{et~al.}\xspace}

\usepackage{times}
\usepackage{helvet}
\usepackage{courier}
\usepackage{amsmath}
\usepackage{algorithm}
\usepackage{algpseudocode}  
\algnewcommand{\LineComment}[1]{\State $\triangleright$ #1}
\usepackage{multirow}

\usepackage{amsfonts}
\usepackage{color}

\usepackage{booktabs} 

\usepackage{graphicx}
\usepackage{subfigure}
\usepackage{amsthm}
\usepackage{verbatim}

\theoremstyle{definition}
\newtheorem{definition}{Definition}[section]
\newtheorem{theorem}{Theorem}
\newtheorem{lemma}{Lemma}

\newcommand{\dvnebold}{\textbf{DVNE}}

\usepackage{amssymb}
\theoremstyle{remark}
\newtheorem*{remark}{Remark}

\newtheorem{property}{Property}

\begin{document}
\title{Deep Learning for Learning Graph Representations}
\author{
	Wenwu Zhu,~~~~~~~Xin Wang,~~~~~~~Peng Cui \\
	Tsinghua University \\
	\{wwzhu,xin\_wang,cuip\}@tsinghua.edu.cn
}
%

%
\maketitle

\begin{abstract}
Mining graph data has become a popular research topic in computer science and has been widely studied in both academia and industry
given the increasing amount of network data in the recent years.
However, the huge amount of network data has posed great challenges for efficient analysis.
This motivates the advent of graph representation which maps the graph into a low-dimension vector space, keeping
original graph structure and supporting graph inference.
The investigation on efficient representation of a graph has profound theoretical significance and important realistic meaning, 
we therefore introduce some basic ideas in graph representation/network embedding as well as some representative models in this chapter.

\end{abstract}

{\bf Keywords: } Deep Learning, Graph Representation, Network Embedding

\section{Introduction}
Many real-world systems, such as Facebook/Twitter social systems, DBLP author-citation systems and roadmap transportation systems etc., 
can be formulated in the form of graphs or networks, making analyzing these systems equivalent to mining their corresponding graphs or networks.
Literature on mining graphs or networks has two names: graph representation
and network embedding. 
We remark that {\it graph} and {\it network} all refer to the same kind of structure, although each of them may have its own terminology,
e.g., a {\it vertice} and an {\it edge} in a graph v.s. a {\it node} and a {\it link} in a network.
Therefore we will exchangeably use graph representation and network embedding without further explanations in the remainder of this chapter.
The core of mining graphs/networks relies heavily on properly representing graphs/networks, making representation learning on graphs/networks
a fundamental research problem in both academia and industry. Traditional representation approaches represent graphs directly based on
their topologies, resulting in many issues including sparseness, high computational complexities etc., which actuates the advent of machine learning
based methods that explore the latent representations capable of capturing extra information in addition to topological structures for graphs in vector space. 
As such, the ability to find ``good'' latent representations for graphs plays an important role in accurate graph representations.
However, learning network representations faces challenges as follows: 
\begin{enumerate}
	\item \textbf{High non-linearity}. As is claimed by Luo~\etal~\cite{luo2011cauchy}, the network has highly non-linear underlying structure. 
	Accordingly, it is a rather challenging work to design a proper model to capture the \textit{highly non-linear} structure.
	\item \textbf{Structure-preserving}. With the aim of supporting network analysis applications, preserving the network structure is required for network embedding. However, the underlying structure of the network is quite \textit{complex} \cite{shaw2009structure}. In that the similarity of vertexes depends on both the local and global network structure, it is a tough problem to preserve the local and global structure simultaneously.
	\item \textbf{Property-preserving}. Real-world networks are normally very complex, their formation and evolution are accompanied with various properties such as uncertainties and dynamics. Capturing these properties in graph representation is of significant importance and poses great challenges. 
	\item \textbf{Sparsity}. Real-world networks are often too \textit{sparse} to provide adequate observed links for utilization, consequently causing unsatisfactory performances \cite{perozzi2014deepwalk}.
\end{enumerate}

Many network embedding methods have been put forward, which adopt shallow models like IsoMAP \cite{tenenbaum2000global}, Laplacian Eigenmaps (LE) \cite{belkin2003laplacian} and Line \cite{tang2015line}. However, on account of the limited representation ability \cite{bengio2009learning}, it is challenging for them to capture the highly nonlinear structure of the networks\cite{tian2014learning}. As \cite{zhuang2011two} stated, although some methods adopt kernel techniques \cite{vishwanathan2010graph}, they still belong to shallow models, incapable of capturing the highly non-linear structure well.
On the other hand, the success of deep learning in handling non-linearity brings us great opportunities for accurate representations in latent vector space. 
One natural question is that can we utilize deep learning to boost the performance of graph representation learning? The answer is yes, and we will
discuss some recent advances in combining deep learning techniques with graph representation learning in this chapter. 
Our discussions fall in two categories of approaches: deep structure-oriented approaches and deep property-oriented approaches. 
For structure-oriented approaches, we include three methods as follows.
\begin{itemize}
\item Structural deep network embedding (SDNE)~\cite{wang2016structural} that focuses on preserving high order proximity.
\item Deep recursive network embedding (DRNE)~\cite{tu2018deep} that focuses on preserving global structure.
\item Deep hyper-network embedding (DHNE)~\cite{tu2018structural} that focuses on preserving hyper structure.
\end{itemize}
For property-oriented approaches, we discuss:
\begin{itemize}
\item Deep variational network embedding (DVNE)~\cite{zhu2018deep} that focuses on the uncertainty property.
\item Deeply transformed high-order Laplacian Gaussian process (DepthLGP) based network embedding~\cite{ma2018depthlgp} 
that focuses on the dynamic (i.e., out-of-sample) property.
\end{itemize}


\section*{\it{\LARGE{Deep Structure-oriented Methods}}}

\section{High Order Proximity Preserving Network Embedding}
\label{sec:SDNE_introduction}

Deep learning, as a powerful tool capable of learning complex structures of the data \cite{bengio2009learning} through efficient representation, 
has been widely adopted to tackle a large number of tasks related to image~\cite{krizhevsky2012imagenet}, 
audio~\cite{hinton2012deep} and text \cite{socher2013recursive} etc. 
To preserve the high order proximity as well as capture the \textbf{highly non-linear} structure, Wang~\etal~\cite{wang2016structural} propose to learn vertex representations for networks by resorting to autoencoder, motivated by the recent success of deep neural networks. 
Concretely, the authors design a multi-layer architecture containing multiple non-linear functions, which 
maps the data into a highly non-linear latent space, thus is able to capture the highly non-linear network structure.

\begin{figure}[htb!]
\centering
\includegraphics[width=0.7\textwidth]{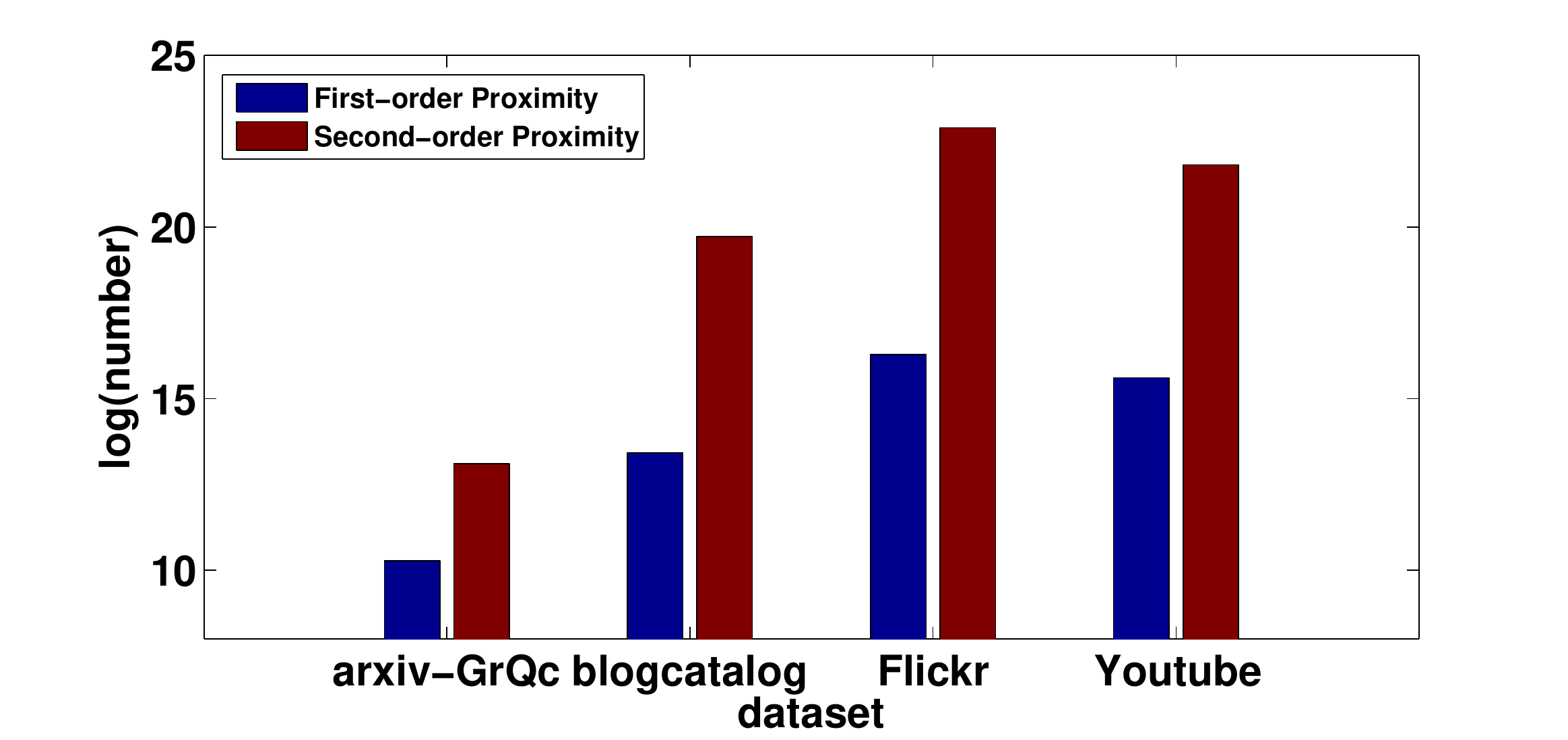}
\caption{ The number of pairs of vertexes which have first-order and second-order proximity in different datasets, figure from~\cite{wang2016structural}.}
\label{pic:res_number}
\end{figure}

So as to resolve the \textbf{structure-preserving} and \textbf{sparsity} problems in the deep models, the authors further put forward a method to jointly mine the first-order and second-order proximity \cite{tang2015line} during the learning process, where the former captures the local network structure, only characterizing the local pairwise similarity between the vertexes linked by edges. Nonetheless, many legitimate links are missing due to the sparsity of the network. Consequently, the first-order proximity alone cannot represent the network structure sufficiently.
Therefore, the authors further advance the second-order proximity, the indication of the similarity among the vertexes' neighborhood structures, to characterize the global structure of networks. 
With the first-order and second-order proximity adopted simultaneously, the model can capture both the local and global network structures well respectively. The authors also propose a semi-supervised architecture to preserve both the local and global network structure in the deep model, where the first-order proximity is exploited as the the supervised information by the supervised component exploits, preserving the local one, while the second-order proximity is reconstructed by the unsupervised component, preserving the global one. Moreover, as is illustrated in Figure \ref{pic:res_number}, there are much more pairs of vertexes having second-order proximity than first-order proximity. Hence, in the light of characterizing the network structure, importing second-order proximity can provide much more information.
In general, for purpose of preserving the network structure, SDNE is capable of mapping the data to a highly non-linear latent space while it is also robust to sparse networks.
To our best knowledge, SDNE is among the first to adopt deep learning structures for network representation learning.

\subsection{Problem Definition}
\begin{definition}[Graph] 
$G=(V,E)$ represents a graph, where $V = \{v_1, ..., v_n\}$ stands for $n$ vertexes and $E = \{e_{i,j}\}_{i,j=1}^n$ stands for the edges. Each edge $e_{i,j}$ is associated with a weight $s_{i,j} \geq 0$ \footnote{Negative links exist in signed network, but only non-negative links are considered here.}. For $v_i$ and $v_j$ without being linked by any edge, $s_{i,j} = 0$. Otherwise, $s_{i,j} = 1$ for unweighted graph and $s_{i,j} > 0$ for weighted graph.
\end{definition}

The goal of network embedding is mapping the graph data into a lower-dimensional latent space. Specifically, each vertex is mapped to a low-dimensional vector so that the network computation can be directly done in that latent space. As mentioned before, preserving both local and global structure is essential. First, the first-order proximity able to characterize the local network structure, is defined as follows.
\begin{definition}[First-Order Proximity] 
The first-order proximity represents the pairwise proximity between vertexes. For a vertex pair, first-order proximity between $v_i$ and $v_j$ is positive if $s_{i,j} > 0$ and $0$ otherwise.
\end{definition}

Spontaneously, network embedding is requisite to preserve the first-order proximity for the reason that it means that two vertexes linked by an observed edge in real-world networks are always similar. For instance, if a paper is cited by another, they are supposed to have some common topics. Nonetheless, real-world datasets often have such high sparsity that only a small portion is the observed links. Many vertexes with similarity are not linked by any edges in the networks. Accordingly, it is not sufficient to only capture the first-order proximity, which is why the second-order proximity is introduced as follows to characterize the global network structure.

\begin{definition}[Second-Order Proximity] 
The second-order proximity of a vertex pair represents the proximity of the pair's neighborhood structure. Let $\mathcal{N}_u=\{s_{u,1},...,s_{u,|V|}\}$ stand for the first-order proximity between $v_u$ and other vertexes. Second-order proximity is then decided by the similarity of $\mathcal{N}_u$ and $\mathcal{N}_v$.
\end{definition}

Intuitively, the second-order proximity presumes two vertexes to be similar if they share many common neighbors. In many fields, such an assumption has been proved reasonable \cite{dash2008context,jin2001structure}. For instance, if two linguistics words always have similar contexts, they will usually be similar \cite{dash2008context}. People sharing many common friends tend to be friends \cite{jin2001structure}. It has been demonstrated that the second-order proximity is a good metric for defining the similarity between vertex pairs even without being linked by edges \cite{liben2007link}, which can also highly improve the richness of vertex relationship consequently. Thus, taking the second-order proximity into consideration enables the model to capture the global network structure and relieve the sparsity problem as well.

To preserve both the local and global structure when in network embedding scenarios, we now focus on the problem of how to integrate the first-order and second-order proximity simultaneously, the definition of which is as follows.

\begin{definition}[Network Embedding] 
Given a graph $G=(V,E)$, the goal of network embedding is learning a mapping function $f: v_i \longmapsto{\mathbf{y_i} \in \mathbb{R}^d}$, where $d \ll |V|$. The target of the function is to enable the similarity between $\mathbf{y_i}$ and $\mathbf{y_j}$ to preserve the \textit{first-order} and \textit{second-order} proximity of $v_i$ and $v_j$ explicitly.
\end{definition}

\subsection{The SDNE Model}
\begin{figure}[htb!]
\centering
\includegraphics[width=0.7\textwidth]{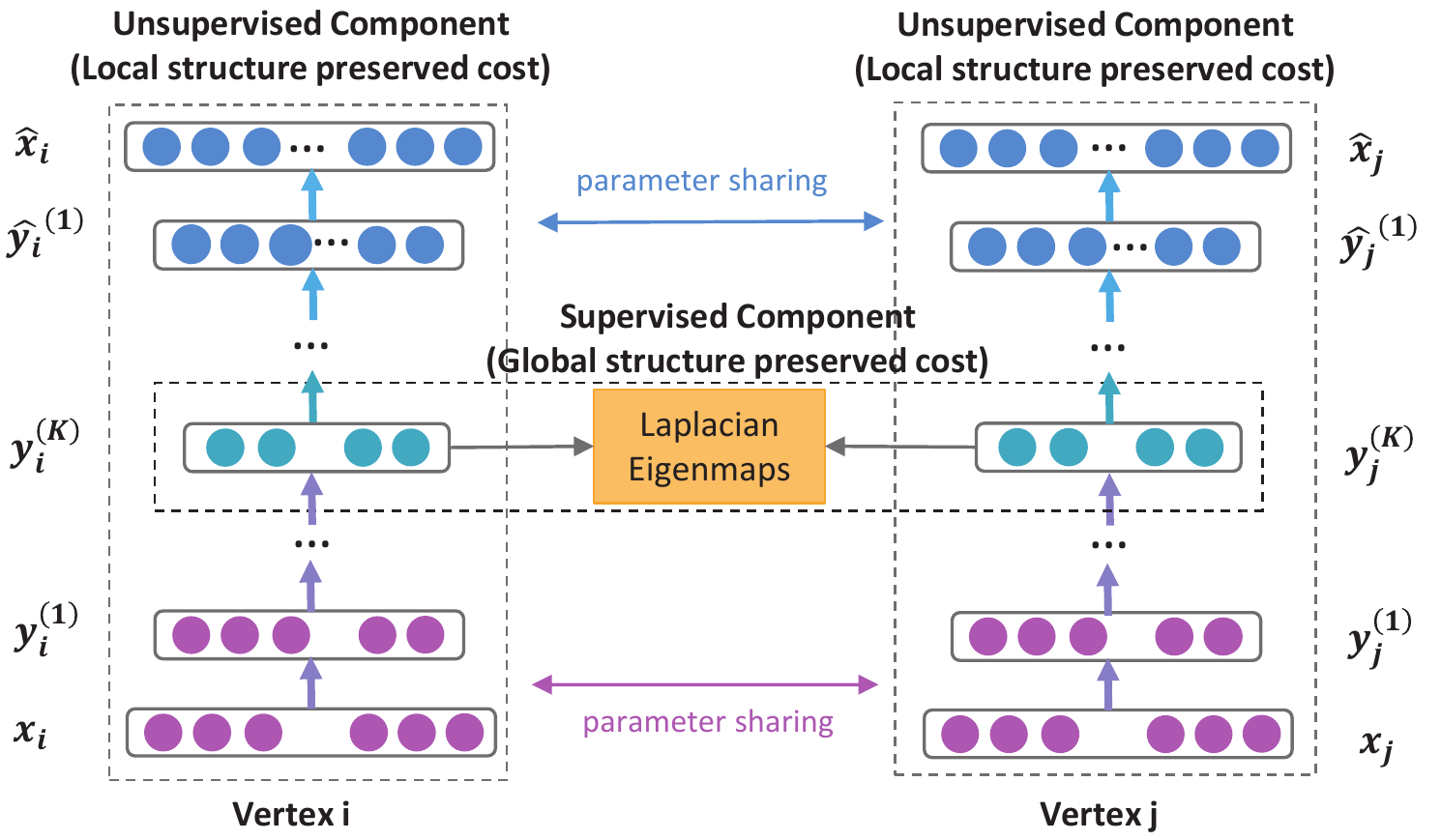}
\label{pic:MDBN}
\caption{Framework of the \emph{SDNE} model, figure from~\cite{wang2016structural}.}
\label{pic:graphencoder}
\end{figure}

This section discusses the semi-supervised SDNE model for network embedding, the framework of which is illustrated in Figure \ref{pic:graphencoder}. Specifically, for purpose of characterizing the highly non-linear network structure, the authors put forward a deep architecture containing numerous non-linear mapping functions to transform the input into a highly non-linear latent space. Moreover, for purpose of exploiting both the first-order and second-order proximity, a semi-supervised model is adopted, aiming to resolve the problems of structure-preserving and sparsity. 
We are able to obtain the neighborhood of each vertex. Hence, to preserve the second-order proximity by the method of reconstructing the neighborhood structure of every vertex, the authors project the unsupervised component. At the same time, for a small portion of vertex pairs, obtaining their pairwise similarities (i.e. the first-order proximity) is also possible. Thus, the supervised component is also adopted to exploit the first-order proximity as the supervised information for refining the latent representations. By optimizing these two types of proximity jointly in the semi-supervised model proposed, \emph{SDNE} is capable of preserving the highly-nonlinear local and global network structure well and is also robust when dealing with sparse networks. 


\subsubsection{Loss Functions}
We give definition of some notations and terms, before defining the loss functions, to be used later in Table \ref{tab:notations-sdne}. Note that $\hat{}$ symbol above the parameters stands for decoder parameters.
\begin{table}[htb!]
\centering\small\caption{Terms and Notations.}\label{tab:notations-sdne}
\begin{tabular}{|c|c|}
\hline
Symbol & Definition\\
\hline\hline
$n$ & number of vertexes\\
$K$ & number of layers\\
$S = \{\mathbf{s}_1,...,\mathbf{s}_n\}$ & the adjacency matrix for the network \\
$X = \{\mathbf{x}_i\}_{i=1}^n$, $\hat{X}=\{\mathbf{\hat{x}}_i\}_{i=1}^n$ & the input data and reconstructed data\\
$Y^{(k)} = \{\mathbf{y}_i^{(k)}\}_{i=1}^n$ & the $k$-th layer hidden representations \\
$W^{(k)}$, $\hat{W}^{(k)}$ & the $k$-th layer weight matrix \\
$\mathbf{b^{(k)}}$, $\mathbf{\hat{b}^{(k)}}$ & the $k$-th layer biases\\
$\theta = \{W^{(k)},\hat{W}^{(k)},\mathbf{b}^{(k)},\mathbf{\hat{b}}^{(k)}\}$ & the overall parameters \\
\hline
\end{tabular}
\end{table}

To begin with, we describe how the second-order proximity is exploited by the unsupervised component in order to preserve the global network structure.

The second-order proximity refers to the similarity of the neighborhood structure of a  vertex pair. Therefore, to capture it properly, we need to consider the neighborhood of each vertex when modeling. For a network $G = (V,E)$, its adjacency matrix $S$ containing $n$ instances $\mathbf{s}_1,...,\mathbf{s}_n$ can be easily obtained. For each instance $\mathbf{s}_i=\{s_{i,j}\}_{j=1}^n$, $s_{i,j}>0$ iff there exists a link between $v_i$ and $v_j$. Therefore, $\mathbf{s}_i$ represents the neighborhood structure of the vertex $v_i$, i.e., $S$ involves the information of each vertex's neighborhood structure. Based on $S$, conventional deep autoencoder \cite{salakhutdinov2009semantic} is extended for purpose of preserving the second-order proximity.

We briefly review the key ideas of deep autoencoder to be self-contained. A deep autoencoder is an unsupervised model consisting of an encoder and decoder. The two parts both consist of numerous non-linear functions, while the encoder maps the input data to the representation space and the decoder maps the representation space to reconstruction space. The hidden representations for each layer are defined as follows, given the input $\mathbf{x}_i$\footnote{Here the authors use sigmoid function $\sigma(x)=\frac{1}{1+exp(-x)}$ as the non-linear activation function}.
\begin{eqnarray}
\begin{aligned}
\label{eqn:calrepre}
\mathbf{y}_i^{(1)} &= \sigma(W^{(1)}\mathbf{x}_i+\mathbf{b}^{(1)}) \\
\mathbf{y}_i^{(k)} &= \sigma(W^{(k)}\mathbf{y}_i^{(k-1)}+\mathbf{b}^{(k)}), k=2,...,K . \\
\end{aligned}
\end{eqnarray}

With $\mathbf{y}_i^{(K)}$ obtained, the output $\hat{\mathbf{x}}_i$ can be obtained through the reversion the encoder's computation process. The goal of an autoencoder is minimizing the reconstruction error between input and output, to achieve which the following loss function can be defined.
\begin{eqnarray}
\begin{aligned}
\label{eqn:DMAE}
\mathcal{L} &= {\sum_{i=1}^n\|\hat{\mathbf{x}}_i-\mathbf{x}_i\|_2^2} .
\end{aligned}
\end{eqnarray}

As proven by \cite{salakhutdinov2009semantic}, the reconstruction formula can smoothly characterize the data manifolds and thus preserve it implicitly though not explicitly. Consider this case: if the adjacency matrix $S$ is inputted into the autoencoder, i.e., $\mathbf{x}_i = \mathbf{s}_i$, the reconstruction process will output the similar latent representations for the vertexes with similar neighborhood structures, as each instance $\mathbf{s}_i$ captures the neighborhood structure of the vertex $v_i$.

However, owing to some specific characteristics of the networks, such a reconstruction process cannot fit our problem straightforward. In the networks, some links can be observed, while many other legitimate links cannot. It suggests that while the links among vertexes do indicate their similarity, having no links does not necessarily indicate dissimilarity between the vertexes. In addition, the number of 0 elements in $S$ is far more than that of non-zero elements due to the sparsity of networks. Thus, by directly inputting $S$ to the conventional autoencoder, there is a tendency of reconstructing the 0 elements in $S$, which does not fit out expectation. Therefore, with the help of the revised objective function as follows, SDNE imposes more penalty to the reconstructing error for the non-zero elements than that for 0 elements:
\begin{eqnarray}
\begin{aligned}
\mathcal{L}_{2nd} &= \sum_{i=1}^n\|\mathbf{(\hat{x}}_i-\mathbf{x}_i)\odot \mathbf{b_i}\|_2^2\\
&= \|(\hat{X}-X) \odot B\|_{F}^2 ,
\label{eqn:loss_secondorder}
\end{aligned}
\end{eqnarray}
where $\odot$ stands for the Hadamard product, $\mathbf{b_i}=\{b_{i,j}\}_{j=1}^n$. If $s_{i,j}=0$, $b_{i,j}=1$, else $b_{i,j}=\beta>1$. Now by inputting $S$ to the revised deep autoencoder, vertexes with alike neighborhood structure will have close representations in the latent space, which is guaranteed by the reconstruction formula. In another word, reconstructing the second-order proximity among vertexes enables the unsupervised component of SDNE to preserve the global network structure.


As explained above, preserving the global and local network structure are both essential in this task. For purpose of representing the local network structure, the authors use first-order proximity, the supervised information for constraining the similarity among the latent representations of vertex pairs. Hence, the supervised component is designed to exploit this first-order proximity. The following definition of loss function is designed for this target.\footnote{To simplify the notations, network representations $Y^{(K)}=\{\mathbf{y}^{(K)}_i\}_{i=1}^n$ are denoted as $Y=\{\mathbf{y}_i\}_{i=1}^n$ by the authors.}.

\begin{eqnarray}
\begin{aligned}
\mathcal{L}_{1st} &= \sum_{i,j=1}^ns_{i,j}\|\mathbf{y}_i^{(K)}-\mathbf{y}_j^{(K)}\|_2^2 \\
                  &= \sum_{i,j=1}^ns_{i,j}\|\mathbf{y}_i-\mathbf{y}_j\|_2^2 .
\label{eqn:loss_firstorder}
\end{aligned}
\end{eqnarray}

A penalty is incurred by this objective function when similar vertexes are mapped far away in the latent representation space, which borrows the idea of Laplacian Eigenmaps \cite{belkin2003laplacian}. Other works on social networks \cite{jamali2010matrix} also use the similar idea. SDNE differs from these methods on the fact that it incorporates this idea into the deep model to embed the linked-by-edge vertexes close to each other in the latent representation space, preserving the first-order proximity consequently.

For purpose of preserving the first-order and second-order proximity simultaneously, Eq.\eqref{eqn:loss_firstorder} and Eq.\eqref{eqn:loss_secondorder} is combined by SDNE through jointly minimizing the objective functions as follows.
\begin{eqnarray}
\begin{aligned}
\mathcal{L}_{mix}&=\mathcal{L}_{2nd}+\alpha \mathcal{L}_{1st}+ \nu \mathcal{L}_{reg} \\
&=\|(\hat{X}-X) \odot B\|_{F}^2
+\alpha\sum_{i,j=1}^ns_{i,j}\|\mathbf{y}_i-\mathbf{y}_j\|_2^2+\nu \mathcal{L}_{reg},
\label{eqn:loss_mix}
\end{aligned}
\end{eqnarray}
where $\mathcal{L}_{reg}$ stands for an $\mathcal{L}$2-norm regularizer term as follows, aiming to avoid overfitting:
\begin{equation}
\mathcal{L}_{reg}=\frac{1}{2}\sum_{k=1}^K(\|W^{(k)}\|_F^2+\|\hat{W}^{(k)}\|_F^2). \nonumber
\end{equation}


\subsubsection{Optimization}
For purpose of optimizing the model mentioned above, we minimize $\mathcal{L}_{mix}$ as a function of $\theta$. Specifically, the critical step is computing the partial derivative of $\partial{\mathcal{L}_{mix}}/\partial{\hat{W}^{(k)}}$ and $\partial{\mathcal{L}_{mix}}/\partial{W}^{(k)}$ with the following detailed mathematical form:
\begin{eqnarray}
\begin{aligned}
\frac{\partial{\mathcal{L}_{mix}}}{\partial{\hat{W}^{(k)}}} &= \frac{\partial{\mathcal{L}_{2nd}}}{\partial{\hat{W}^{(k)}}} + \nu\frac{\partial{\mathcal{L}_{reg}}}{\partial{\hat{W}^{(k)}}}\\
\frac{\partial{\mathcal{L}_{mix}}}{\partial{W^{(k)}}} &= \frac{\partial{\mathcal{L}_{2nd}}}{\partial{W}^{(k)}} + \alpha\frac{\partial{\mathcal{L}_{1st}}}{\partial{W}^{(k)}}+\nu\frac{\partial{\mathcal{L}_{reg}}}{\partial{W}^{(k)}},k=1,...,K .
\label{eqn:derivative}
\end{aligned}
\end{eqnarray}

First, we focus on $\partial{\mathcal{L}_{2nd}}/\partial{\hat{W}}^{(K)}$, which can be rephrased as follows.
\begin{equation}
\frac{\partial{\mathcal{L}_{2nd}}}{\partial{\hat{W}^{(K)}}} = \frac{\partial{\mathcal{L}_{2nd}}}{\partial{\hat{X}}}\cdot\frac{\partial{\hat{X}}}{\partial{\hat{W}^{(K)}}}.
\end{equation}

In light of Eq.\eqref{eqn:loss_secondorder}, for the first term we have
\begin{equation}
\frac{\partial{\mathcal{L}_{2nd}}}{\partial{\hat{X}}} = 2(\hat{X}-X)\odot B .
\end{equation}

The computation of the second term $\partial{\hat{X}}/\partial{\hat{W}}$ is simple because $\hat{X}=\sigma(\hat{Y}^{(K-1)}\hat{W}^{(K)}+\hat{b}^{(K)})$, with which $\partial{\mathcal{L}_{2nd}}/{\partial{\hat{W}^{(K)}}}$ is available. Through back-propagation method, we can iteratively acquire $\partial{\mathcal{L}_{2nd}}/{\partial{\hat{W}}^{(k)}},k=1,...K-1$ and $\partial{\mathcal{L}_{2nd}}/{\partial{W}^{(k)}},k=1,...K$. 

Next, we move on to the partial derivative of $\partial{\mathcal{L}_{1st}}/\partial{W}^{(k)}$. The loss function of $\mathcal{L}_{1st}$ can be rephrased as follows.
\begin{equation}
\mathcal{L}_{1st} = \sum_{i,j=1}^ns_{i,j}\|\mathbf{y}_i-\mathbf{y}_j\|_2^2 = 2tr(Y^TLY) ,
\label{eqn:loss_secondorder_matrix}
\end{equation}
where $L = D - S$, $D\in\mathbb{R}^{n\times n}$ stands for the diagonal matrix where $D_{i,i} = \sum\nolimits_{j}{s_{i,j}}$.

Next, we center upon the computation of $\partial{\mathcal{L}_{1st}}/{\partial{W}^{(K)}}$ first:
\begin{equation}
\frac{\partial{\mathcal{L}_{1st}}}{\partial{W}^{(K)}} = \frac{\partial{\mathcal{L}_{1st}}}{\partial{Y}}\cdot\frac{\partial{Y}}{\partial{W}^{(K)}} .
\end{equation}

Because $Y = \sigma(Y^{(K-1)}W^{(K)}+b^{(K)})$, the computation of the second term $\partial{Y}/{\partial{W}^{(K)}}$ is simple. For $\partial{\mathcal{L}_{1st}}/{\partial{Y}}$, we hold
\begin{equation}
\frac{\partial{\mathcal{L}_{1st}}}{\partial{Y}} = 2(L+L^T)\cdot Y .
\end{equation}

Likewise, the calculation of partial derivative of $\mathcal{L}_{1st}$ can be finally finished through back-propagation.

All the partial derivatives of the parameters have been acquired now. After implementing parameter initialization, we can optimize the deep model proposed above with stochastic gradient descent. It is worth mentioning that owing to its high nonlinearity, the model may fall into many local optimum in the parameter space. Hence, the authors adopt Deep Belief Network as a method of pretraining at first \cite{hinton2006fast} to find a good region in parameter space, which has been proved to be an fundamental way of initialization for deep learning architectures\cite{erhan2010does}. Algorithm~\ref{alg:SDNE} presents the complete algorithm.
\begin{algorithm}[htb]
\caption{Training Algorithm for \emph{SDNE}}
\label{alg:SDNE}
\begin{algorithmic}[1]
\Require the network $G=(V,E)$ with adjacency matrix $S$, the parameters $\nu$ and $\alpha$
\Ensure Representations Y of the network and updated Parameters: $\theta$
\State Pretrain the model with deep belief network, obtaining the parameters $\theta=\{\theta^{(1)},...,\theta^{(K)}\}$
\State $X = S$
\Repeat
\State Apply Eq.\eqref{eqn:calrepre} to calculate $\hat{X}$ and $Y=Y^{K}$ given $X$ and $\theta$.
\State $\mathcal{L}_{mix}(X;\theta) = \|(\hat{X}-X) \odot B\|_{F}^2
+2\alpha tr(Y^TLY)+\nu \mathcal{L}_{reg}$.
\State Utilize $\partial{L_{mix}}/\partial{\theta}$ to back-propagate throughout the entire network based on Eq.\eqref{eqn:derivative} to calculate the updated parameters of $\theta$.
\Until{convergence}
\State Return the network representations $Y = Y^{(K)}$
\end{algorithmic}
\end{algorithm}

\subsection{Analysis and Discussions on SDNE}

\textbf{New vertexes.} Learning representations for newly arrived vertexes is a practical issue for network embedding. If we know the connections of a new vertex $v_k$ to the existing vertexes, its adjacency vector $\mathbf{x}=\{s_{1,k},...,s_{n,k}\}$ is easy to obtained, where $s_{i,k}$ indicates the similarity between the new vertex $v_k$ and the existing $v_i$. Then $\mathbf{x}$ can be simply fed into the deep model, after which we can calculate the representations for $v_k$ with the trained parameters $\theta$. For such a process, the time complexity is $O(1)$. Nonetheless, SDNE fails when there are no connections between existing vertexes and $v_k$ in the network. 

\textbf{Training Complexity.} The complexity of SDNE is $O(ncdI)$, where $n$ stands for the number of vertexes, $c$ represents the average degree of the network, $d$ stands for the maximum dimension of the hidden layer, and $I$ represents the number of iterations. $d$ is often related to the dimension of the embedding vectors but not $n$, while $I$ is also independent of $n$. In real-world applications, the parameter $c$ can be regarded as a constant. For instance, the maximum number of a user's friends is always bounded \cite{tian2014learning}in a social network. Meanwhile, $c=k$ in a top-k similarity graph. Thus, $cdI$ is also independent of $n$. So the total training complexity is actually linear to $n$ .
\section{Global Structure Preserving Network Embedding}
\label{sec:DRNE_introduction}

As is discussed, one fundamental problem of network embedding is how to preserve the vertex similarity in an embedding space, i.e., two vertexes should have the similar embedding vectors if they have similar local structures in the original network. To quantify the similarity among vertexes in a network, the most common one among multiple methods is {\em structural equivalence}~\cite{leicht2006vertex}, where two vertexes sharing lots of common network neighbors are considered structurally equivalent. Besides, preserving structural equivalence through high-order proximities~\cite{tang2015line,wang2016structural} is the aim of most previous work on network embedding, where network neighbors are extended to high-order neighbors, e.g., direct neighbors and neighbors-of-neighbors, etc.

However, vertexes without any common neighbors can also occupy similar positions or play similar roles in many cases. For instance, two mothers share the same pattern of connecting with several children and one husband (the father). The two mothers do share similar positions or roles although they are not structurally equivalent if they do not share any relatives. These cases lead to an extended definition of vertex similarity known as {\em regular equivalence}:
two regularly equivalent vertexes have network neighbors which are themselves similar (i.e., regularly equivalent)~\cite{rossi2015role}.
As neighbor relationships in a network can be defined recursively, we remark that regularly equivalence is able to reflect the global structure of a network.
Besides, regular equivalence is, apparently, a relaxation of structural equivalence. Structural equivalence promises regular equivalence, while the reverse direction does not hold. Comparatively, regular equivalence is more capable and flexible of covering a wider range of network applications with relation to node importance or structural roles, but is largely ignored by previous work on network embedding.

For purpose of preserving global structure and regular equivalence in network embedding, i.e., two nodes of regularly equivalence should have similar embeddings, a simple way is to explicitly compute the regular equivalence of all vertex pairs and preserve the similarities of corresponding embeddings of nodes to approximate their regular equivalence. Nevertheless, due to the high complexity in computing regular equivalence in large-scale networks, this idea is infeasible. 
Another way is to replace regular equivalence with simpler graph theoretic metrics, such as centrality measures. Although many centrality measures have been proposed to characterize the importance and role of a vertex, it is still difficult to learn general and task-independent node embeddings because one centrality can only capture a specific aspect of network role. In addition, some centrality measures, e.g., betweeness centrality, also bear high computational complexity. Thus, how to efficiently and effectively preserve regular equivalence in network embedding is still an open problem.
	
\begin{figure}
	\hspace*{-0.1cm}
	\centering
	\includegraphics[scale=0.5]{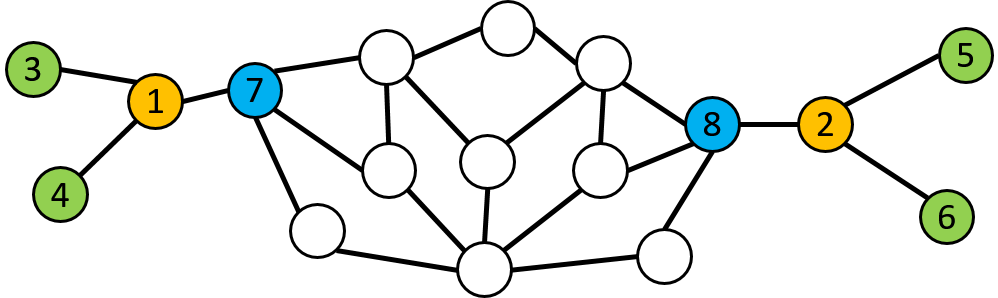}
	\caption{A simple graph to illustrate the rationality of why recursive embedding can preserve regular equivalence. The regularly equivalent nodes share the same color, figure from~\cite{tu2018deep}.}
	\label{fig:graph}
\end{figure}
	
Fortunately, the recursive nature in the definition of regular equivalence enlightens Tu~\etal~\cite{tu2018deep} to learn network embedding in a recursive way, i.e., the embedding of one node can be aggregated by its neighbors' embeddings. In one recursive step (Figure~\ref{fig:graph}), if nodes 7 and 8, 4 and 6, 3 and 5 are regularly equivalent and thus have similar embeddings already, then nodes 1 and 2 would also have similar embeddings, resulting in their regular equivalence. It is this idea that inspires the design of the {\em Deep Recursive Network Embedding} (DRNE) model~\cite{tu2018deep}. In specific, the neighbors of a node are transformed into an ordered sequence and a layer normalized LSTM (Long Short Term Memory networks)~\cite{hochreiter1997long} is proposed to aggregate the embeddings of neighbors into the embedding of the target node in a non-linear way.

\subsection{Preliminaries and Definitions}
	In this section, we discuss the Deep Recursive Network Embedding (DRNE) model whose framework is demonstrated in Figure~\ref{fig:drne_framework}. Taking node $0$ in Figure~\ref{fig:drne_framework} as an example, we sample three nodes $1, 2, 3$ from its neighborhoods and sort them by degree as a sequence $(3, 1, 2)$. We use the embeddings of the neighborhoods sequence $\mathbf{X}_3, \mathbf{X}_1, \mathbf{X}_2$ as input, aggregating them by a layer normalized LSTM to get the assembled representation $h_T$. By reconstructing the embedding $\mathbf{X}_0$ of node $0$ with the aggregated representation $h_T$, the embedding $\mathbf{X}_0$ can preserve the local neighborhood structure. On the other hand, we use the degree $d_0$ as weak supervision information of centrality and put the aggregated representation $h_T$ into a multilayer perceptron (MLP) to approximate degree $d_0$. The same process is conducted for each other node. When we update the embedding of the neighborhoods $\mathbf{X}_3, \mathbf{X}_1, \mathbf{X}_2$, it will affect the embeddings $\mathbf{X}_0$. Repeating this procedure by updating the embeddings iteratively, the embeddings $\mathbf{X}_0$ can contain structural information of the whole network. 
	
	\begin{figure}
		\centering
		\includegraphics[scale=0.5]{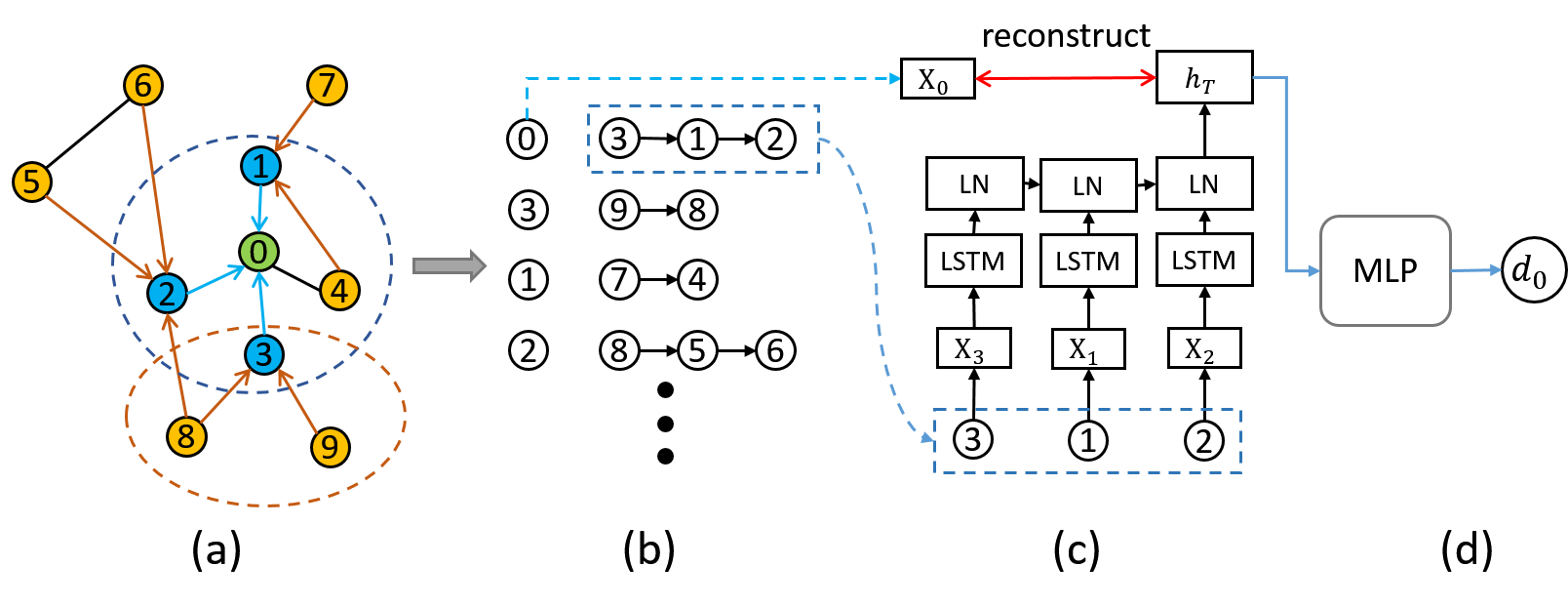}
		\caption{{Deep Recursive Network Embedding (DRNE). (a) Sampling neighborhoods. (b) Sorting neighborhoods by their degrees. (c) Layer-normalized LSTM to aggregate embeddings of neighboring nodes into the embedding of the target node. $X_i$ is the embedding of node $i$ and LN means layer normalization. (d) A Weakly guided regularizer. Figure from~\cite{tu2018deep}.}}
		\label{fig:drne_framework}
	\end{figure}
	
	Given a network $G = (V, E)$, where $V$ stands for the set of nodes and $E \in V \times V$ edges. For a node $v \in V$, $\mathcal{N}(v) = \{u | (v , u) \in E\}$ represents the set of its neighborhoods. The learned embeddings are defined as $\mathbf{X} \in \mathbb{R}^{|V| \times k}$ where $k$ is the dimension and $\mathbf{X}_v \in \mathbb{R}^k$ is the embedding of node $v$. The degree of node $v$ is defined as $d_v = |\mathcal{N}(v)|$ and function $I(x)=1$ if $x \ge 0$ otherwise 0. The formal definition of structural equivalence and regular equivalence is given as follows.
	
	\begin{definition}[Structural Equivalence]
		\label{def:se}
		We denote $s(u) = s(v)$ if nodes $u$ and $v$ are structurally equivalent. Then $s(u) = s(v)$ if and only if $\mathcal{N}(u) = \mathcal{N}(v)$.
	\end{definition}
	
	\begin{definition}[Regular Equivalence]
		\label{def:re}
		We denote $r(u) = r(v)$ if nodes $u$ and $v$ are regularly equivalent. Then $r(u) = r(v)$ if and only if $\{r(i) | i \in \mathcal{N}(u)\} = \{r(j) | j \in \mathcal{N}(u)\}$.
	\end{definition}

\subsection{The DRNE Model}
\subsubsection{Recursive Embedding}
	In light of Definition~\ref{def:re}, DRNE learns the embeddings of nodes recursively: the embedding of a target node can be approximated by aggregating the embeddings of its neighbors, So we can use the following loss function:
	\begin{equation}
	\label{equ:rec}
	\mathcal{L}_{1} = \sum_{v \in V}||\mathbf{X}_v-Agg(\{\mathbf{X}_u| u \in \mathcal{N}(v)\})||_F^2,
	\end{equation}
	where $Agg$ is the function of aggregation. In each recursive step, the local structure of the neighbors of the target node can be preserved by its learned embedding. Therefore, the learned node embeddings can incorporate the structural information in a global sense by updating the learned representations iteratively, which consists with the definition of regular equivalence.

	As for the aggregating function, DRNE utilizes the layer normalized Long Short-Term Memory(ln-LSTM)~\cite{ba2016layer} due to the highly nonlinearity of the underlying structures of many real networks~\cite{luo2011cauchy}. LSTM is an effective model for modeling sequences, as is known to all. Nonetheless, in networks, the neighbors of a node have no natural ordering. Here the degree of nodes is adopted as the criterion of sorting neighbors in an ordered sequence for the reason that taking degree as measure for neighbor ordering is the most efficient and that degree is a crucial part in many graph-theoretic measures, notably those relating to structural roles, e.g. Katz~\cite{nathan2017dynamic} and PageRank~\cite{page1999pagerank}.
	
	Suppose $\{X_1$, $X_2$, ..., $X_t$, ..., $X_T\}$ are the embeddings of the ordered neighbors. At each time step $t$, the hidden state $h_t$ is a function of its previous hidden state $h_{t-1}$ and input embedding $X_t$, i.e., $h_t = LSTMCell(h_{t-1}, X_t)$. The information of hidden representation $h_t$ will become increasingly abundant while the embedding sequence is processed by LSTM Cell recursively from 1 to $T$. Therefore, $h_T$ can be treated as the aggregation of the representation from neighbors. In addition, LSTM with gating mechanisms is effective in learning long-distance correlations in long sequences. In the structure of LSTM, the input gate along with old memory decides what new information to be stored in memory, the forget gate decides what information to be thrown away from the memory and output gate decides what to output based on the memory. Specifically, $LSTMCell$, the LSTM transition equation, can be written as follows.
	\begin{align}
	\mathbf{f}_t &= \sigma(\mathbf{W}_f\cdot[\mathbf{h}_{t-1}, \mathbf{X}_t]+\mathbf{b}_f), \label{equ:lstm:f} \\
	\mathbf{i}_t &= \sigma(\mathbf{W}_i\cdot[\mathbf{h}_{t-1}, \mathbf{X}_t]+\mathbf{b}_i), \label{equ:lstm:i} \\
	\mathbf{o}_t &= \sigma(\mathbf{W}_o\cdot[\mathbf{h}_{t-1}, \mathbf{X}_t]+\mathbf{b}_o), \label{equ:lstm:o} \\
	\tilde{\mathbf{C}_t} &= tanh(\mathbf{W}_C\cdot[\mathbf{h}_{t-1}, \mathbf{X}_t]+\mathbf{b}_C), \label{equ:lstm:tC}\\
	\mathbf{C}_t &= \mathbf{f}_t*\mathbf{C}_{t-1}+\mathbf{i}_t*\tilde{\mathbf{C}_t}, \label{equ:cell} \\
	\mathbf{h}_t &= \mathbf{o}_t*tanh(\mathbf{C}_t), \label{equ:lstm:h}
	\end{align}
	where $\sigma$ stands for the sigmoid function, $*$ and $\cdot$ represent element-wise product and matrix product respectively, $C_t$ represents the cell state, $\mathbf{i}_t$, $\mathbf{f}_t$ and $\mathbf{o}_t$ are input gate, forget gate and output gate respectively. $\mathbf{W}_*$ and $\mathbf{b}_*$ are the parameters to be learned.
	
	Moreover, the Layer Normalization~\cite{ba2016layer} is adopted in DRNE for purpose of avoiding the problems of exploding or vanishing gradients~\cite{hochreiter1997long} when long sequences are taken as input. The layer normalized LSTM makes re-scaling all of the summed input invariant, resulting in much more stable dynamics. In particular, with the extra normalization as follows, it re-centers and re-scales the cell state $C_t$ after Eq.~\eqref{equ:cell}.
	\begin{equation}
	C'_t = \frac{g}{\Sigma_t}*(C_t-\mu_t),
	\end{equation}
	where $\mu_t = 1/k\sum_{i=1}^{k}C_{ti}$ and $\Sigma_t = \sqrt{1/k\sum_{i=1}^{k}(C_{ti}-\mu_t)^2}$ are the
	mean and variance of $C_t$ , and $g$ is the gain parameter scaling the normalized activation.
	
\subsubsection{Regularization}
	Without any other constraints, $\mathcal{L}_1$ defined in Eq.~\eqref{equ:rec} represents the recursive embedding process according to Definition~\ref{def:re}. Its expressive power is so strong that we can obtain multiple solutions to satisfy the recursive process defined above. The model take a risk of degenerating to a trivial solution: all the embeddings become $0$. 
	For purpose of avoiding this degeneration, DRNE takes node degree as the weakly guiding information, i.e., imposing a constraint that the learned embedding of a target node should be capable of approximating its degree. Consequently, the following regularizer is designed:
	\begin{equation}
	\label{equ:reg}
	\mathcal{L}_{reg} = \sum_{v \in V} \left\| \log(d_v+1)-MLP(Agg(\{\mathbf{X}_u | u \in \mathcal{N}(v)\})) \right\|_F^2,
	\end{equation}
	where the degree of node $v$ is denoted by $d_v$ and $MLP$ stands for a single-layer multilayer perceptron taking the rectified linear unit (ReLU)~\cite{glorot2011deep} as activation function, defined as $ReLU(x) =max(0, x)$. DRNE minimizes the overall objective function by combining reconstruction loss in Eq.~\eqref{equ:rec} and the regularizer in Eq.~\eqref{equ:reg}:
	\begin{equation}
	\mathcal{L} = \mathcal{L}_1+\lambda\mathcal{L}_{reg},
	\end{equation}
	where $\lambda$ is the balancing weight for regularizer. Note that degree information is not taken as the supervisory information for network embedding here. Alternatively, it is finally auxiliary to avoid degeneration. Therefore, the value of $\lambda$ should be set small.
	
	\textbf{Neighborhood Sampling.} The node degrees usually obey a heavy-tailed distribution~\cite{eom2015tail} in real networks, i.e., the majority of nodes have very small degrees while a minor number of nodes have very high degrees. Inspired by this phenomenon, DRNE downsamples the neighbors of those large-degree nodes before feeding them into the ln-LSTM to improve the efficiency. In specific, an upper bound for the number of neighbors $S$ is set. When the number of neighbors exceeds $S$, the neighbors are downsampled into $S$ different nodes. An example on the sampling process is shown in Figure \ref{fig:drne_framework} (a) and (b). In networks obeying power-law, more unique structural information are carried by the large-degree nodes than the common small-degree nodes. Therefore, a biased sampling strategy is adopted by DRNE to retain the large-degree nodes by setting $P(v) \propto d_v$, i.e., the probability of sampling neighbor node $v$ being proportional to its degree.
	
\subsubsection{Optimization}
    For purpose of optimizing DRNE, we need to minimize the total loss $\mathcal{L}$ as a function of the embeddings $X$ and the neural network parameters set $\theta$. These parameters are optimized by Adam~\cite{kingma2014adam}. The BackPropagation Through Time (BPTT) algorithm~\cite{werbos1990backpropagation} estimates the derivatives. At the beginning of the training, the learning rate $\alpha$ for Adam is initialized to 0.0025. Algorithm \ref{algorithm-drne} demonstrates the whole algorithm.
	
	\begin{algorithm}
	\caption{Deep Recursive Network Embedding}
	\label{algorithm-drne}
	\begin{algorithmic}[1]
	\Require the network $\mathbf{G} = (V, E)$
	\Ensure the embeddings $\mathbf{X}$,  updated neural network parameters set $\theta$
	\State initial $\theta$ and $\mathbf{X}$ by random process
	\While{the value of objective function do not converge} 
	\For{a node $v \in E$}
	\State downsampling $v$'s neighborhoods if its degree exceeds $S$
	\State sort the neighborhoods by their degrees
	\State fixed aggregator function, calculate partial derivative $\partial\mathcal{L}/\partial\mathbf{X}$ to update embeddings $\mathbf{X}$
	\State fixed embeddings, calculate partial derivative $\partial\mathcal{L}/\partial\theta$ to update $\theta$
	\EndFor
	\EndWhile
	\State obtain the node representations $X$
	\end{algorithmic}
	\end{algorithm}
	
\subsubsection{Theoretical Analysis}
	It can be theoretically proved that the resulted embeddings from DRNE can reflect several common and typical centrality measures closely related to regular equivalence. In the following process of proof, the regularizer term in Eq.~\eqref{equ:reg} for eliminating degeneration is ignored without loss of generality.
	
	\begin{theorem}
		\label{the:main}
		Eigenvector centrality~\cite{bonacich2007some}, PageRank centrality~\cite{page1999pagerank} and degree centrality are three optimal solutions of DRNE, respectively.
	\end{theorem}
	
	
	\begin{lemma}
		\label{lemma:1}
		For any computable function, there exists a finite recurrent neural network (RNN)~\cite{mikolov2010recurrent} that can compute it.
	\end{lemma}
	\begin{proof}
		This is a direct consequence of Theorem 1 in \cite{siegelmann1995computational}.
	\end{proof}
	
	\begin{theorem}
		\label{the:2}
		If the centrality $C(v)$ of node $v$ satisfies that $C(v) = \\ \sum_{u \in \mathcal{N}(v)}F(u)C(u)$ and $F(v) = f(\{F(u), u\in\mathcal{N}(v)\})$ where $f$ is any computable function, then $C(v)$ is one of the optimal solutions for DRNE.
	\end{theorem}
	\begin{proof}
		For brevity, let us suppose that LSTM takes linear activation for all the activation function. This lemma is proved by showing that there exists a parameter setting $\{\mathbf{W}_f, \mathbf{W}_i, \mathbf{W}_o, \mathbf{W}_C, \mathbf{b}_f, \mathbf{b}_i, \mathbf{b}_o, \mathbf{b}_C\}$ in Eq. \eqref{equ:lstm:f}, \eqref{equ:lstm:i}, \eqref{equ:lstm:o}, \eqref{equ:lstm:tC} such that the node embedding $\mathbf{X}_u = [F(u), C(u)]$ is a fixed point. In fact, this parameter settings can be directly constructed. Suppose $\mathbf{W}_{a, i}$ denotes the i-th row of $\mathbf{W}_a$. With the input sequence $\{[F(u), C(u)], u\in\mathcal{N}(v) \}$, set $\mathbf{W}_{f,2}$ and $\mathbf{W}_{o,2}$ as $[0, 0]$, $\mathbf{W}_{i,2}$ as $[1, 0]$, $\mathbf{W}_{C,2}$ as $[0, 1]$, $\mathbf{b}_{f,2}$ and $\mathbf{b}_{o,2}$ as $1$, $\mathbf{b}_{i,2}$ and $\mathbf{b}_{C,2}$as $0$, then we can easily get $\mathbf{h}_{t,2} = \mathbf{o}_{f,2}*\mathbf{C}_{t,2} = \mathbf{C}_{t,2} = \mathbf{f}_{t,2}*\mathbf{C}_{t-1,2}+\mathbf{i}_{t,2}*\tilde{\mathbf{C}}_{t,2} = \mathbf{C}_{t-1,2}+F(t)*C(t)$. Hence, $\mathbf{h}_{T,2}=\sum_{u \in \mathcal{N}(v)}F(u)C(u) = C(v)$ where $T$ is the length of the input sequence. Additionally, by Lemma \ref{lemma:1}, there exists a parameter setting $\{\mathbf{W}_f', \mathbf{W}_i', \mathbf{W}_o', \mathbf{W}_C', \mathbf{b}_f', \mathbf{b}_i', \mathbf{b}_o', \mathbf{b}_C'\}$ to approximate $f$. By set $\mathbf{W}_{f,1}$ as $[\mathbf{W}_f', 0]$, $\mathbf{W}_{o, 1}$ as $[\mathbf{W}_o', 0]$ and so on, we can get that $\mathbf{h}_{T,1} = f(\{F(u), u\in\mathcal{N}(v)\}) = F(v)$. Therefore $\mathbf{h}_T = [F(v), C(v)]$ and the node embedding $\mathbf{X}_v = [F(v), C(v)]$ is a fixed point. The proof is now completed.
	\end{proof}
	
	We can easily conclude that eigenvector centrality, PageRank centrality, degree centrality satisfy the condition of Theorem~\ref{the:2} by the definitions of centralities in Table \ref{tab:theory} with $(F(v), f(\{x_i\}))$, completing the proof for Theorem~\ref{the:main}.
	
\begin{table}
\centering
\small
	\caption{Definition of centralities. }
	\label{tab:theory}	
		\begin{tabular}{|c|c|c|c|}
			\toprule
			Centrality & Definition $C(v)$ & $F(v)$ & $f(\{x_i\})$ \\
			\midrule
			Eigenvector & $1/\lambda*\sum_{u\in \mathcal{N}(v)}C(u)$ & $1/\lambda$ & mean \\
			\midrule
			PageRank & $\sum_{u\in \mathcal{N}(v)}1/d_u*C(u)$ & $1/d_v$ & $1/(\sum{I(x_i)})$ \\
			\midrule
			Degree & $d_v=\sum_{u\in \mathcal{N}(v)} I(d_u)$ & $1/d_v$  & $1/(\sum{I(x_i)})$ \\
			\bottomrule
	\end{tabular}
\end{table}
	
	Based on Theorem~\ref{the:main}, such a parameter setting of DRNE exists for any graph that the resulted embeddings are able to be one of the three centralities. This shows such expressive power of DRNE that different aspects of regular-equivalence-related network structural information are captured.
	
\subsubsection{Analysis and Discussions}
This section presents the out-of-sample extension and the complexity analysis.

\paragraph{Out-of-sample Extension}
	For a node $v$ newly arrived, we can feed the embeddings of its neighbors directly into the aggregating function to get the aggregated representation, i.e., the embedding of the new node through Eq.~\eqref{equ:rec}, if we know its connections to the existing nodes. Such a procedure has a complexity of $O(d_vk)$, where $d_v$ stands for the degree of node $v$ and $k$ represents the length of embeddings and . 
	
\paragraph{Complexity Analysis}
	For a single node $v$ in each iteration during the training procedure, the complexity of gradients calculation and parameters updating is $O(d_vk^2)$, where $k$ stands for the length of embeddings abd $d_v$ represents the degree of node $v$. The aforementioned sampling process keeps $d_v$ from exceeding the bound $S$. Therefore, the total training complexity is $O(|V|Sk^2I)$ where $I$ stands for the number of iterations. $k$, the length of embeddings, usually takes a small number (e.g. 32, 64, 128). The upper bound $S$ is 300 in DRNE. The number of iterations $I$ normally takes a small number which is independent with $|V|$. Hence, the total time complexity of training procedure is actually linear to the number of nodes $|V|$.
\section{Structure Preserving Hyper Network Embedding}
\label{sec:DHNE_introduction}

Conventional pairwise networks are the scenarios of most network embedding methods, where each edge connects only a pair of nodes. Nonetheless, the relationships among data objects are much more complicated in real world applications, and they typically go beyond pairwise. For instance, Jack purchasing a coat with nylon material forms a high-order relationship $\langle$Jack, coat, nylon$\rangle$. 
A network designed to capture this kind of high-order node relationship is often known as a hyper-network.


For purpose of analyzing hyper-network, expanding them into traditional pairwise networks and then applying the analytical pairwise-network-based algorithms is a typical idea. Star expansion~\cite{agarwal2006higher} (Figure~\ref{fig:hypergraph} (d)) and clique expansion~\cite{sun2008hypergraph} (Figure~\ref{fig:hypergraph} (c)) are two representative techniques of this category. For star expansion, a hypergraph is changed into a bipartite graph where each hyperedge is represented by an instance node linking to the original nodes contained by it. For clique expansion, a hyperedge is expanded as a clique. 
These methods make an assumption that the hyperedges are {\em decomposable} either implicitly or explicitly. In other words, if a set of nodes is treated as a hyperedge, then any subset of nodes contained by this hyperedge can constitute another hyperedge. This assumption is reasonable in homogeneous hyper-networks, because the constitution of hyperedges are caused by the latent similarity among the concerned objects, e.g. common labels, in most cases. Nonetheless, when it comes to the heterogeneous hyper-network embedding, it is essential to resolve the new demand as follows.

\begin{enumerate}
	\item \textbf{Indecomposablity}: In heterogeneous hyper-networks, the hyperedges are often indecomposable. In the circumstances, a node set in a hyperedge has a strong inner relationship, whereas the nodes in its subset does not necessarily have. As an instance, in a recommendation system which has $\langle$user, movie, tag$\rangle$ relationships, the relationships of $\langle$user, tag$\rangle$ are often not strong. This phenomenon suggests that using those traditional expansion methods to simply decompose hyperedges does not make sense.
	
	\item \textbf{Structure Preserving}:
	The observed relationships in network embedding can easily preserve local structures. Nevertheless, many existing relationships are not observed owing to the sparsity of networks, when preserving hyper-network structures with only local structures is not sufficient. Some global structures like the neighborhood structures are employed to address this problem. Thus, how to simultaneously capture and preserve both global and local structures in hyper-networks still remains an open problem.
\end{enumerate}


Tu~\etal~\cite{tu2018structural} put forward a deep hyper-network embedding (DHNE) model to deal with these challenges.
To resolve the \textbf{Indecomposablity} issue, an indecomposable tuplewise similarity function is designed. The function is straightforward defined over the universal set of the nodes contained by a hyperedge, ensuring that the subsets of it are not contained in network embedding process. They provide theoretical proof that the indecomposable tuplewise similarity function cannot be linear. Consequently, they implement the tuplewise similarity function as a deep neural network with a non-linear activation function added, making it highly non-linear. To resolve the \textbf{Structure Preserving} issue, a deep autoencoder is designed to learn node representations by reconstructing neighborhood structure, making sure that the nodes which have alike neighborhood structures also have alike embeddings. To simultaneously address the two issues, the deep autoencoder are jointly optimized with the tuplewise similarity function.


\begin{figure}
	\centering
	\includegraphics[width=0.8\linewidth]{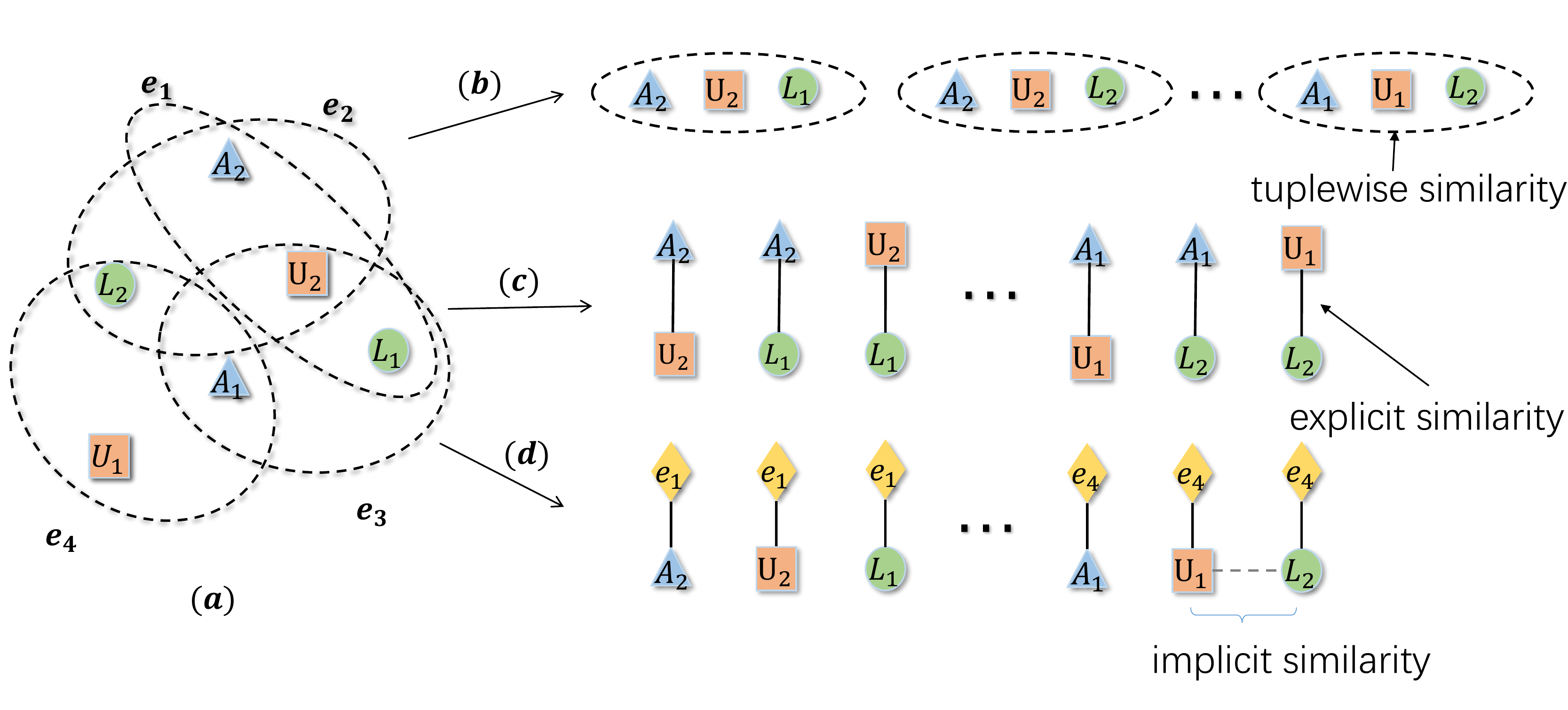}
	\caption{\small (a) One example of a hyper-network. (b) DHNE. (c) The clique expansion. (d) The star expansion. DHNE models the hyperedge by and large, preserving the tuplewise similarity. In the case of clique expansion, each hyperedge is expanded into a clique, where each node pair is explicitly similar. In the case of the star expansion, each node of a hyperedge is linked to a new node representing the origin hyperedge, each node pair of which is implicitly similar in that they are linked to the same node. Figure from~\cite{tu2018structural}.}
	\label{fig:hypergraph}
\end{figure}

\subsection{Notations and Definitions}
The key notations used by DHNE is illustrated in Table~\ref{tab:notations}. 

\begin{table}
	\caption{Notations.}
	\label{tab:notations}
    \centering
    \small
	\begin{tabular}{|c|c|}
		\toprule
		Symbols & Meaning \\
		\midrule
		$T$ & number of node types \\
		$\mathbf{V} = \{\mathbf{V}_t\}_{t=1}^T$ & node set \\
		$\mathbf{E} = \{(\emph{v}_1, \emph{v}_2, ..., \emph{v}_{n_i})\}$ & hyperedge set \\ 
		$\mathbf{A}$ &  adjacency matrix of hyper-network \\
		$\mathbf{X}_i^{\emph{j}}$ &  embedding of node $i$ with type $\emph{j}$ \\
		$\mathcal{S}(\mathbf{X}_1, \mathbf{X}_2,..., \mathbf{X}_N)$ & $N$-tuplewise similarity function \\
		$\mathbf{W}^{(i)}_j$ & the i-th layer weight matrix with type j \\
		$\mathbf{b}^{(i)}_j$ & the i-th layer biases with type j \\
		\bottomrule
	\end{tabular}
\end{table}

\begin{definition}[Hyper-network]
	One \textbf{hyper-network} is a hypergraph $\mathbf{G} = (\mathbf{V}, \mathbf{E})$ where the set of nodes $V$ belongs to $T$ types $\mathbf{V} = \{\mathbf{V}_t\}_{t=1}^T$ and each hyperedge of the set of edges $\mathbf{E}$ may have more than two nodes $\mathbf{E} = \{\emph{E}_i=(\emph{v}_1, \emph{v}_2, ..., \emph{v}_{n_i})\} (n_i \geq 2)$. The hyper-network degenerates to a network when each hyperedge has only 2 nodes. The definition of the type of edge $\emph{E}_i$ is the combination of types of all the nodes in the edge. If $T \geq 2$, we define the hyper-network as \textbf{heterogeneous hyper-network}.
\end{definition}

In order to obtain embedding in a hyper-network, it is necessary to preserve an indecomposable tuplewise relationship. The authors give a definition to the indecomposable structures as the first-order proximity of the hyper-network.

\begin{definition}[The First-order Proximity of Hyper-network]
	\label{def:first-order}
	\textbf{The first-order proximity} of a hyper-network measures the N-tuplewise similarity between nodes. If there exists a hyperedge among any $N$ vertexes $\emph{v}_1, \emph{v}_2, ..., \emph{v}_{N}$, the first-order proximity of these $N$ vertexes is defined as 1. Note that this implies no first-order proximity for any subsets of these N vertexes.
\end{definition}

In the real world, the first-order proximity suggests the indecomposable similarity among several entities. Moreover, real world networks are always sparse and incomplete, thus it is not sufficient to only consider first-order proximity for learning node embeddings. To address this issue, we need to consider higher order proximity. To capture the global structure, the authors then propose the definition of the second-order proximity of hyper-networks.

\begin{definition}[The Second-order Proximity of Hyper-network]
	\textbf{The second-order Proximity} of a hyper-network measures the proximity of two nodes concerning their neighborhood structures. For any node $\emph{v}_i \in \emph{E}_i$, $\emph{E}_i/{\emph{v}_i}$ is defined as a \textbf{neighborhood} of $\emph{v}_i$. If $\emph{v}_i$'s neighborhoods $\{\emph{E}_i/{\emph{v}_i} ~for~any~\emph{v}_i \in \emph{E}_i\}$ are similar to $\emph{v}_j$'s, then $\emph{v}_i$'s embedding $\mathbf{x}_i$ should be similar to $\emph{v}_j$'s embedding $\mathbf{x}_j$.
\end{definition}

For example, in Figure \ref{fig:hypergraph}(a), $\emph{A}_1$'s neighborhood set is $\{(\emph{L}_2, \emph{U}_1), (\emph{L}_1, \emph{U}_2)\}$ and $\emph{A}_2$'s neighborhood set is $\{(\emph{L}_2, \emph{U}_2), (\emph{L}_1, \emph{U}_2)\}$. Thus $\emph{A}_1$ and $\emph{A}_2$ are second-order similar due to sharing common neighborhood $(\emph{L}_1, \emph{U}_2)$.

\subsection{The DHNE Model}

This section presents the Deep Hyper-Network Embedding (DHNE) model, the framework of which is illustrated in Figure \ref{fig:framework}.

\subsubsection{Loss function}
For purpose of preserving the first-order proximity of hyper-networks, an $N$-tuplewise similarity measure is required in the embedding space. Such a measure should meet the requirement that when a hyperedge exists among $N$ vertexes, the $N$-tuplewise similarity of them is supposed to be large and vice versa.

\begin{property}
	\label{property}
	Let $\mathbf{X}_i$ denote the embedding of node $v_i$ and $\mathcal{S}$ as $N$-tuplewise similarity function.
	\begin{itemize}
		\item if $(\emph{v}_1, \emph{v}_2, ..., \emph{v}_N) \in \mathbf{E}$, $\mathcal{S}(\mathbf{X}_1, \mathbf{X}_2, .., \mathbf{X}_N)$ is supposed to be large (larger than a threshold $l$ without loss of generality). 
		\item if $(\emph{v}_1, \emph{v}_2, ..., \emph{v}_N) \notin \mathbf{E}$, $\mathcal{S}(\mathbf{X}_1, \mathbf{X}_2, .., \mathbf{X}_N)$ is supposed to be small (smaller than a threshold $s$ without loss of generality).
	\end{itemize}
\end{property}

\begin{figure} 
	\centering
	\includegraphics[width=0.8\linewidth]{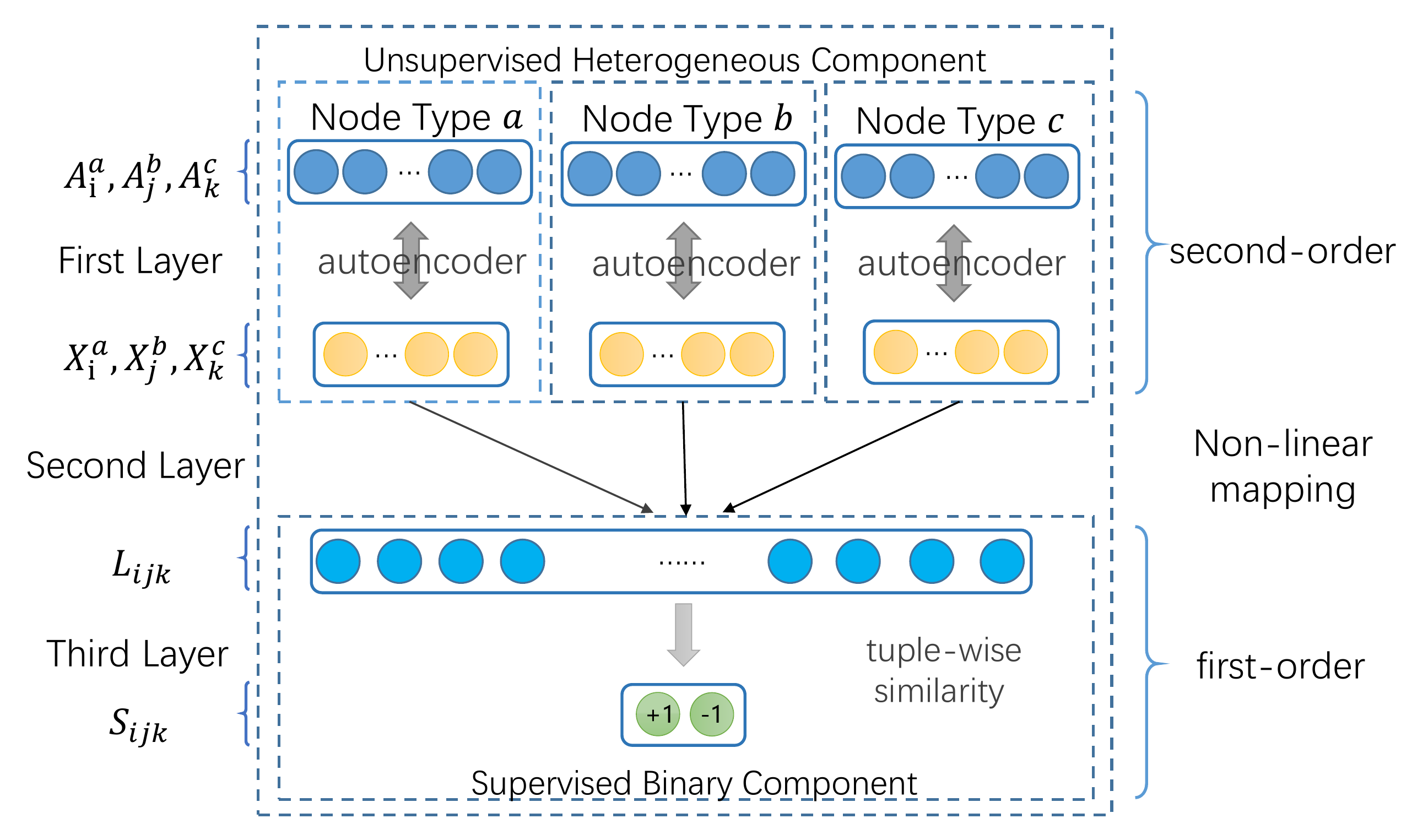}
	\caption{{\small Framework of Deep Hyper-Network Embedding (DHNE), figure from~\cite{tu2018structural}.}}
	\label{fig:framework}
\end{figure}

DHNE employs a data-dependent $N$-tuplewise similarity function and mainly focuses on hyperedges with uniform length $N = 3$, which is not difficult to extend to $N > 3$.
The authors also propose a theorem to show that a linear tuplewise similarity function is not able to satisfy Property \ref{property}.

\begin{theorem}
	Linear function $\mathcal{S}(\mathbf{X}_1, \mathbf{X}_2,..., \mathbf{X}_N) = \sum_i \mathbf{W}_i\mathbf{X}_i$ cannot satisfy Property \ref{property}.
	\label{theorem-dhne}
\end{theorem}

\begin{proof}
	Using the counter-evidence method, let us presume that theorem~\ref{theorem-dhne} is not true, i.e., the linear function $\mathcal{S}$ satisfies Property~\ref{property}. Consider the counter example with 3 types of nodes where each type has 2 clusters (with the id 1 and 0). Hence, a hyperedge exists iff 3 nodes from different types have the same cluster id. We take the notation of $\mathbf{Y}_i^{j}$ to stand for embeddings of nodes with type $j$ in cluster $i$. We hold the following inequations by Property \ref{property}:
	\begin{align}
		\mathbf{W}_1\mathbf{Y}_0^1+\mathbf{W}_2\mathbf{Y}_0^2+\mathbf{W}_3\mathbf{Y}_0^3 > l \label{equ:ce1}&\\
		\mathbf{W}_1\mathbf{Y}_1^1+\mathbf{W}_2\mathbf{Y}_0^2+\mathbf{W}_3\mathbf{Y}_0^3 < s \label{equ:ce2}&\\
		\mathbf{W}_1\mathbf{Y}_1^1+\mathbf{W}_2\mathbf{Y}_1^2+\mathbf{W}_3\mathbf{Y}_1^3 > l \label{equ:ce3}&\\
		\mathbf{W}_1\mathbf{Y}_1^0+\mathbf{W}_2\mathbf{Y}_1^2+\mathbf{W}_3\mathbf{Y}_1^3 < s \label{equ:ce4}&.
	\end{align}
	By combining Eq.\eqref{equ:ce1}\eqref{equ:ce2}\eqref{equ:ce3}\eqref{equ:ce4}, we get $\mathbf{W}_1*(\mathbf{Y}_0^1-\mathbf{Y}_1^1) > l-s$ and $\mathbf{W}_1*(\mathbf{Y}_1^1-\mathbf{Y}_0^1) > l-s$. This is contradictory to our assumption and thus finishes the proof.
\end{proof}

Theorem~\ref{theorem-dhne} demonstrates that N-tuplewise similarity function $\mathcal{S}$ are supposed to be non-linear, which motivates DHNE to model the similarity by a multilayer perceptron. This contains two parts, illustrated separately in the 2nd layer and 3rd layer of Figure \ref{fig:framework}, where the 2nd layer is a fully connected layer whose activation functions are non-linear. Inputted with the embeddings $(\mathbf{X}_i^\emph{a}, \mathbf{X}_j^\emph{b}, \mathbf{X}_k^\emph{c})$ of 3 nodes $(\emph{v}_i, \emph{v}_j, \emph{v}_k)$, they can be concatenated and mapped non-linearly to a common latent space $\mathbf{L}$ where the joint representation is shown as follows.

\begin{equation}
\mathbf{L}_{ijk} = \sigma(\mathbf{W}^{(2)}_\emph{a}*\mathbf{X}_i^\emph{a}+
\mathbf{W}^{(2)}_\emph{b}*\mathbf{X}_j^\emph{b}+
\mathbf{W}^{(2)}_\emph{c}*\mathbf{X}_k^\emph{c}+\mathbf{b}^{(2)}), 
\end{equation}
where $\sigma$ stands for the sigmoid function.
Finally, the latent representation $\mathbf{L}_{ijk}$ is mapped to a probability space in the 3rd layer to obtain the similarity:

\begin{equation}
\mathbf{S}_{ijk} \equiv \mathcal{S}(\mathbf{X}^\emph{a}_i, \mathbf{X}^\emph{b}_j, \mathbf{X}^\emph{c}_k) = \sigma(\mathbf{W}^{(3)}*\mathbf{L}_{ijk}+\mathbf{b}^{(3)}).
\end{equation}

Hence, the combination of the second and third layers can get a non-linear tuplewise similarity measure function $\mathcal{S}$ as we hoped. For purpose of making $\mathcal{S}$ satisfy Property \ref{property}, we can write the following objective function.

\begin{equation}
\label{equ:first-order}
\mathcal{L}_1 = -(\mathbf{R}_{ijk}\log \mathbf{S}_{ijk}+(1-\mathbf{R}_{ijk})\log(1-\mathbf{S}_{ijk})),
\end{equation}
where $\mathbf{R}_{ijk}$ is defined as $1$ if there is a hyperedge between $\emph{v}_i$, $\emph{v}_j$ and $\emph{v}_k$ and $0$ otherwise. According to the objective function, it is not difficult to point out that if $\mathbf{R}_{ijk} = 1$, the similarity $\mathbf{S}_{ijk}$ is supposed to be large and vice versa. That is to say, the first-order proximity is successfully preserved. 

The design of the first layer in Figure \ref{fig:framework} aims to preserve the second-order proximity, which measures the similarity of neighborhood structures. Here, to characterize the neighborhood structure, the authors define the adjacency matrix of hyper-network. Specifically, they define a $|\mathbf{V}|*|\mathbf{E}|$ incidence matrix $\mathbf{H}$ with elements $\mathbf{h}(\emph{v}, \emph{e}) = 1$ if $\emph{v} \in \emph{e}$ and 0 otherwise to denote a hypergraph $\mathbf{G} = (\mathbf{V}, \mathbf{E})$. $d(v) = \sum_{\emph{e} \in \mathbf{E}} \mathbf{h}(\emph{v}, \emph{e})$ stands for the degree of a vertex $\emph{v} \in \mathbf{V}$. Let $\mathbf{D}_v$ stand for the diagonal matrix containing the vertex degree, then the adjacency matrix $\mathbf{A}$ of hypergraph $\mathbf{G}$ can be defined as $\mathbf{A} = \mathbf{H}\mathbf{H}^T-\mathbf{D}_v$, where $\mathbf{H}^T$ is the transpose of $\mathbf{H}$. Here, each element of adjacency matrix $\mathbf{A}$ stands for the concurrent times between two nodes, while the i-th row of $\mathbf{A}$ demonstrates the neighborhood structure of vertex $\emph{v}_i$. To make best of this information, DHNE utilizes an autoencoder~\cite{lecun2015deep} model to preserve the neighborhood structure and feeds it with an adjacency matrix $\mathbf{A}$ as the input feature. The autoencoder consists of two non-linear mapping: an encoder and a decoder, where the encoder maps from feature space $\mathbf{A}$ to latent representation space $\mathbf{X}$, while the decoder from latent representation $\mathbf{X}$ space back to origin feature space $\hat{\mathbf{A}}$, which can be shown as follows.
\begin{align}
	\mathbf{X}_i &= \sigma(\mathbf{W}^{(1)}*\mathbf{A}_i+\mathbf{b}^{(1)}) \label{equ:encoder} \\ 
	\hat{\mathbf{A}}_i &= \sigma(\hat{\mathbf{W}}^{(1)}*\mathbf{X}_i+\hat{\mathbf{b}}^{(1)}). \label{equ:decoder}
\end{align}

The aim of autoencoder is to minimize the reconstruction error between the output and the input, with which process it will give similar latent representations to the nodes with similar neighborhoods, preserving the second-order proximity consequently. Note that the input feature, the adjacency matrix of the hyper-network, is often extremely sparse. To achieve a higher efficiency, DHNE only reconstructs non-zero elements in the adjacency matrix. The following equantion shows the reconstruction error:
\begin{equation}
	||sign(\mathbf{A}_i)\odot(\mathbf{A}_i - \hat{\mathbf{A}}_i)||_F^2,
\end{equation}
where $sign$ stands for the sign function.

Additionally, in a heterogeneous hyper-network, the vertexes usually have various types, the distinct characteristics of which require the model to learn a unique latent space for each of them. Motivated by this idea, DHNE provides each heterogeneous type of entities with an autoencoder model of their own, as is demonstrated in Figure \ref{fig:framework}. The definition of loss function for all types of nodes is as follows.

\begin{equation}
\label{equ:second-order}
\mathcal{L}_2 = \sum_{t} ||sign(\mathbf{A}_i^t)\odot(\mathbf{A}_i^t - \hat{\mathbf{A}}_i^t)||_F^2,
\end{equation}
where t is the index for node types.

For purpose of simultaneously preserving first-order proximity and second-order proximity for a heterogeneous hyper-network, DHNE jointly minimizes the loss function via blending Eq.\eqref{equ:first-order} and Eq.\eqref{equ:second-order}:

\begin{equation}
\label{equ:obj}
\mathcal{L} = \mathcal{L}_1+\alpha\mathcal{L}_2.
\end{equation}

\subsubsection{Optimization}
DHNE adopts stochastic gradient descent (SGD) for optimization, the critical step of which is to compute the partial derivative of parameters $\theta = \{\mathbf{W}^{(i)}, \mathbf{b}^{(i)}, \mathbf{\hat{W}}^{(i)}, \mathbf{\hat{b}}^{(i)}\}_{i=1}^3$. By back-propagation algorithm~\cite{lecun2015deep}, these derivatives can be easily estimated. Note that in most real world networks, there exist only positive relationships, so that the iterative process may degenerate to trivial convergence where all the tuplewise relationships turn out to be similar. In order to resolve this issue, DHNE samples multiple negative edges with the help of a noisy distribution for each edge~\cite{mikolov2013distributed}. The whole algorithm is demonstrated in Algorithm~\ref{algorithm}.

\begin{algorithm}[!t]
	\caption{The Deep Hyper-Network Embedding (DHNE)}
	\label{algorithm}
	\begin{algorithmic}[1]
		\Require the hyper-network $\mathbf{G} = (\mathbf{V}, \mathbf{E})$, the adjacency matrix $\mathbf{A}$ and the parameter $\alpha$
		\Ensure Hyper-network Embeddings $E$ and updated Parameters $\theta = \{\mathbf{W}^{(i)}, \mathbf{b}^{(i)}, \mathbf{\hat{W}}^{(i)}, \mathbf{\hat{b}}^{(i)}\}_{i=1}^3$
		\State Initialize parameters $\theta$ randomly
		\While{the value of objective function has not converged} 
		\State Generate the next batch from the set of hyperedges $\mathbf{E}$
		\State Sample negative hyperedge in a random way
		\State Compute partial derivative $\partial\mathcal{L}/\partial \theta$ with back-propagation to update $\theta$.
		\EndWhile 
	\end{algorithmic}
\end{algorithm}

\subsubsection{Analysis and Discussions}
This section presents the out-of-sample extension and the complexity analysis.
\paragraph{Out-of-sample extension}
For any new vertex $\emph{v}$, it is easy to obtain the adjacency vector by this vertex's connections to other existing vertexes. Hence, the out-of-sample extension problem can be solved by feeding the new vertex $v$'s adjacency vector into the specific autoencoder corresponded with its type and applying Eq.\eqref{equ:encoder} to get its latent representation in embedding space. The time complexity for these steps is $\mathcal{O}(dd_v)$, where $d$ stands for the dimensionality of the embedding space and $d_v$ is the degree of vertex $\emph{v}$.
\paragraph{Complexity analysis}
The complexity of gradients calculation and parameters updating in the training procedure is $O((nd+dl+l)bI)$, where $n$ stands for the number of nodes, $d$ represents the dimension of embedding vectors, $l$ stands for the size of latent layer, $b$ stands for the batch size and $I$ represents the number of iterations. The parameter $l$ is usually correlated with $d$, but independent on $n$ and $I$ also has no connection with $n$. $b$ is normally small. Hence, the time complexity of the training procedure is actually linear to the number of vertexes $n$.

\section*{\it{\LARGE{Deep Property-oriented Methods}}}

\section{Uncertainty-aware Network Embedding}
\label{sec_DVHE_introduction}

Usually, real-world networks, the constitution and evolution of which are full of uncertainties, can be much more sophisticated than we expect. There are many reasons resulting in such uncertainties. For instance, low-degree nodes in a network fail to provide enough information and hence the representations of them are more uncertain than others. For those nodes sharing numerous communities, the potential contradiction among its neighbors might also be larger than others, resulting in uncertainty. 
    Moreover, in social networks, human behaviors are sophisticated, making the generation of edges also uncertain \cite{zang2017long}. Therefore without considering the uncertainties in networks, the information of nodes may become incomplete in latent space, which makes the representations less effective for network analysis and inference.
    Nonetheless, previous work on network embedding mainly represents each node as a single point in lower-dimensional continuous vector space,
    which has a crucial limitation that it can not capture the uncertainty of the nodes.
    Given that the family of Gaussian methods are capable of innately modeling uncertainties \cite{vilnis2014word} and provide various distance functions per object,
    it will be promising to represent a node with a Gaussian distribution so that the characteristics of uncertainties can be incorporated.
    
    As such, to capture the uncertainty of each node during the process of network embedding with Gaussian process, there are several basic requirements.
    First, to preserve the transitivity in networks, the embedding space is supposed to be a metric space.
    Transitivity here is a significant property for networks, peculiarly social networks \cite{holland1972holland}.
    For instance, the possibility of a friend of my friend becoming my friend is much larger than that of a person randomly chosen from the crowd.
    Moreover, the transitivity measures the density of loops of length three (triangles) in networks, crucial for computing clustering coefficient and related attributes \cite{chen2015fast}.
    Importantly, the transitivity in networks can be preserved well on condition that the triangle inequality is satisfied by the metric space.
    Second, the uncertainties of nodes should be characterized by the variance terms so that these uncertainties can be well captured, which means that the variance terms should be explicitly related to mean vectors.
    In other words, the proximities of nodes are supposed to be captured by mean vectors, while the uncertainties of nodes are supposed to be modeled by variance terms.
    Third, network structures such as high-order proximity, which can be used in abundant real-world applications as shown in \cite{ou2016asymmetric}, are also supposed to be preserved effectively and efficiently.

    Zhu~\etal~\cite{zhu2018deep} propose a deep variational model, called \dvnebold, which satisfies the above requirements and learns the Gaussian embedding in the Wasserstein space.
    Wasserstein space \cite{courty2017learning} is a metric space where the learned representations are able to preserve the transitivity for networks well.
    Specifically, the similarity measure is defined as Wasserstein distance, a formidable tool based on the optimal transport theory for comparing data
    distributions with wide applications such as computer vision \cite{bonneel2015sliced} and machine learning \cite{kolouri2017optimal}.
    Moreover, the Wasserstein distance enables the fast computation for Gaussian distributions \cite{givens1984class},
    taking linear time complexity to calculate the similarity between two node representations.
    Meanwhile, they use a variant of Wasserstein autoencoder (WAE) \cite{tolstikhin2017wasserstein} to reflect the relationships between variance and mean terms,
    where WAE is a deep variational model which has the goal of minimizing the Wasserstein distance between the original data distribution and the predicted one.
    In general, via preserving the first-order and second-order proximity, the learned representations by DVNE is capable of well capturing the local and global network structures \cite{wang2016structural,tang2015line} .
    
\subsection{Notations}
    $\mathbf{G}=\{\mathbf{V},\mathbf{E}\}$ stands for a network, where $\mathbf{V}=\{v_1,v_2,...,v_N\}$ is a set of nodes and $N$ is the number of them.
    The set of edges between nodes is denoted as $\mathbf{E}$, where $M = |\mathbf{E}|$ is the number of them.
    Let $\mathbf{Nbrs}_i = \{v_j | (v_i, v_j) \in \mathbf{E} \}$ stand for the neighbors of $v_i$.
    The transition matrix is denoted as $\mathbf{P} \in \mathbb{R}^{N{\times}N}$, where $\mathbf{P}(:,j)$ and $\mathbf{P}(i,:)$ denote its $j^{th}$ column and $i^{th}$ row respectively. $\mathbf{P}(i,j)$ stands for the element at the $i^{th}$ row and $j^{th}$ column.
    Given that an edge links $v_i$ to $v_j$ and node degree of $v_i$ is $d_i$, we set $\mathbf{P}(i,j)$ to $\frac{1}{d_i}$ and zero otherwise.
    Then, $\mathbf{h}_i = \mathcal{N}(\mathbf{\mu}_i,\mathbf{\Sigma}_i)$ is defined as a lower-dimensional Gaussian distribution embedding for node $v_i$, where $\mu_i \in \mathbb{R}^{L}$, $\Sigma_i \in \mathbb{R}^{L{\times}L}$.
    $L$ stands for the embedding dimension, satisfying $L \ll N$.

\subsection{The DVNE Model}
    The DVNE model proposed by Zhu~\etal~\cite{zhu2018deep} is discussed in this section. The framework of DVNE is shown in Figure \ref{fig:fw}.

    \begin{figure}
		\centering
        \includegraphics[width=0.8\linewidth]{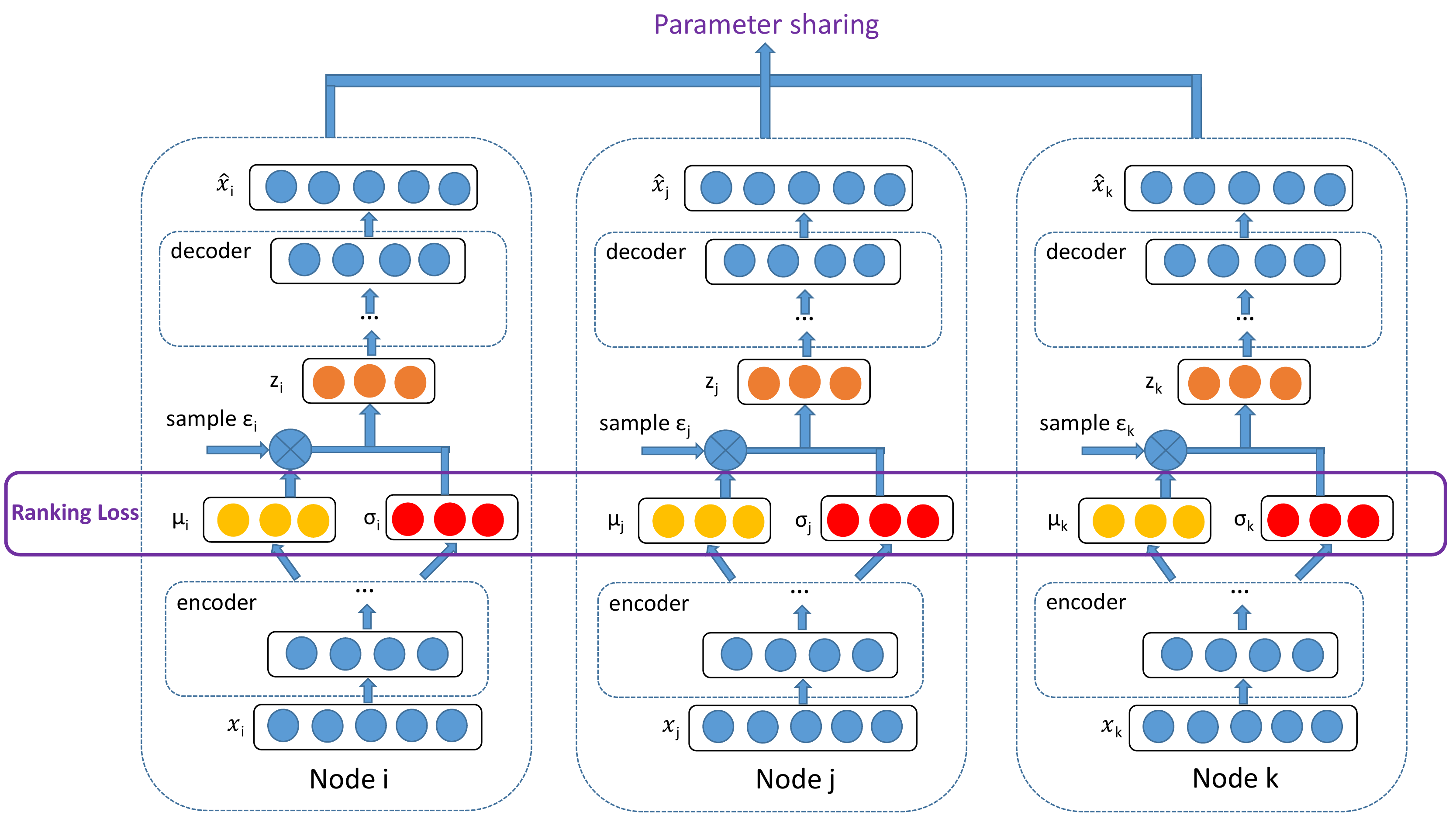}
        \caption{The framework of DVNE,figure from~\cite{zhu2018deep}.}
        \label{fig:fw}
    \end{figure}

\subsubsection{Similarity Measure}
    For purpose of supporting network applications, a suitable similarity measure need to be defined between two node latent representations.
    In DVNE, distributions are adopted to model latent representations, thus the similarity measure here is supposed to be capable of measuring the similarities among different distributions.
    Moreover, the similarity measure is supposed to also simultaneously preserve the transitivity among nodes, since it is a crucial property of networks.
    Wasserstein distance is such an adequate candidate capable of measuring the similarity between two distributions and satisfying the triangle inequality simultaneously\cite{clement2008elementary}, which guarantees its capability of preserving the transitivity of similarities among nodes.

    The definition of the $p^{th}$ Wasserstein distance between two probability measures $\mathcal{\nu}$ and $\mathcal{\mu}$ is
    \begin{equation} \label{eq_wp}
    W_{p} (\nu, \mu)^{p} = \inf \mathbb{E} \big[ d( X , Y )^{p} \big],
    \end{equation}
    where $\mathbb{E}[Z]$ stands for the expectation of a random variable $Z$ and $\inf$ stands for the infimum taken over all the joint distributions of the random variables $X$ and $Y$, the marginals of which are $\nu$ and $\mu$ respectively.
    
    Furthermore, it has been proved that the $p^{th}$ Wasserstein distance can preserve all properties of a metric when $p \geq 1$ \cite{ambrosio2008gradient}.
    The metric should satisfy the non-negativity, the symmetry, the identity of indiscernibles and the triangle inequality \cite{bryant1985metric}.
    In such ways, Wasserstein distance meets the requirement of being a similarity measure for the latent node representations, peculiarly for an undirected network.
    
    Although the computational cost of general-formed Wasserstein distance, which causes the limitation, a closed form solution can be achieved with the $2^{th}$ Wasserstein distance (abbreviated as $W_2$) since Gaussian distributions are used in DVNE
    for the latent node representations. This greatly reduces the computational cost.

    More specifically, DVNE employs the following formula to calculate $W_2$ distance between two Gaussian distributions \cite{givens1984class}:
    \begin{equation} \label{eq_w2}
    \begin{aligned}
         &{dist=W_2(\mathcal{N}(\mu_1,\Sigma_1),\mathcal{N}(\mu_2,\Sigma_2))}\\
         &dist^2=\Vert \mu_1-\mu_2\Vert_2^2 +\mathrm{Tr}(\Sigma_1+\Sigma_2-2(\Sigma_1^{1/2}\Sigma_2\Sigma_1^{1/2})^{1/2})
    \end{aligned}
    \end{equation}

    Furthermore, the $W_2$ distance (a.k.a root mean square bipartite matching distance) has been popularly applied in
    computer graphics \cite{bonneel2011displacement, de2012blue}, computer vision \cite{bonneel2015sliced, courty2017learning} and machine learning \cite{courty2017optimal,cuturi2014fast}, etc.
    DVNE adopts diagonal covariance matrices \footnote{When the covariance matrices are not diagonal, Wang~\etal propose a fast iterative algorithm (i.e., BADMM) to solve the Wasserstein distance \cite{wang2014bregman}.}, thus ${\Sigma_1\Sigma_2=\Sigma_2\Sigma_1}$.
    In the commutative case, the formula \eqref{eq_w2} can be simplified as
    \begin{equation} \label{eq_w2s}
    W_2\big(\mathcal{N}(\mu_1,\Sigma_1);\mathcal{N}(\mu_2,\Sigma_2)\big)^2 =\Vert \mu_1-\mu_2\Vert_2^2 +\Vert\Sigma_1^{1/2}-\Sigma_2^{1/2}\Vert_{F}^2.
    \end{equation}

    According to the above equation, the time complexity of computing $W_2$ distance between two nodes in latent embedding space is linear to $L$, the dimension of embedding.

\subsubsection{Loss Functions}
    First, the first-order proximity needs to be preserved.
    Intuitively, each node connected with $v_i$ is supposed to be of smaller distance to $v_i$ in the embedding space compared to the nodes that have no edges linking $v_i$.
    More specifically, to preserve the first-order proximity, the following pairwise constraints is satisfied by DVNE:
    \begin{equation} \label{eq_s1}
        {W_2}(\mathbf{h}_i, \mathbf{h}_j) <  {W_2}(\mathbf{h}_i, \mathbf{h}_k), \forall{v_i} \in \mathbf{V}, \forall v_j \in \mathbf{Nbrs}_i, \forall v_k \notin \mathbf{Nbrs}_i.
    \end{equation}
    The smaller the $W_2$ distance, the more similar between nodes.
    An energy based learning approach \cite{lecun2006tutorial} is used here to satisfy all pairwise constraints which are defined above.
    The following equation presents the mathematical objective function, with $W_2(\mathbf{h}_i, \mathbf{h}_j)$ standing for the energy between two nodes, $E_{ij} = W_2(\mathbf{h}_i, \mathbf{h}_j)$.
    \begin{equation} \label{eq_l1}
        \mathcal{L}_1 =  \sum_{(i,j,k) \in \mathbf{D}} ({E_{ij}}^2 + exp(-E_{ik})),
    \end{equation}
    where $\mathbf{D}$ stands for the set of all valid triplets given in Eq.\eqref{eq_s1}.
    Therefore, ranking errors are penalized by the energy of the pairs in this objective function, making the energy of negative examples be higher than that of positive examples.

    In order to preserve second-order proximity, transition matrix $\mathbf{P}$ is adopted as the input feature of Wasserstein Auto-Encoders (WAE) \cite{tolstikhin2017wasserstein} to preserve the neighborhood structure and the mathematical relevance of mean vectors and variance terms is also implied.
    More specifically, $\mathbf{P}(i,:)$ demonstrates the neighborhood structure of node $v_i$, and is adopted as the input feature for node $v_i$ to preserve its neighborhood structure.
    The objective of WAE contains the reconstruction loss and the regularizer, where the former loss helps to preserve neighborhood structure and latter guides the encoded training distribution to match the prior distribution.
    Let $P_X$ denotes the data distribution, and $P_G$ denotes the encoded training distribution,
    then the goal of WAE is minimizing Wasserstein distance between $P_X$ and $P_G$.
    The reconstruction cost can be written as
    \begin{equation} \label{eq_reswae}
        D_{WAE}(P_X, P_G) = \inf_{Q(Z|X) \in Q} \mathbb{E}_{P_X} \mathbb{E}_{Q(Z|X)} \big[c(X,G(Z)) \big],
    \end{equation}
    where $Q$ stands for the encoders and $G$ represents the decoders, $X \sim P_X$ and $Z \sim Q(Z|X)$.
    According to \cite{tolstikhin2017wasserstein}, Eq.\eqref{eq_reswae} minimizes the $W_2$ distance between $P_X$ and $P_G$ with $c(x, y) = \Vert x - y \Vert^2_2$.
    Taking the sparsity of transition matrix $\mathbf{P}$ into consideration, DVNE is centered on non-zero elements in $\mathbf{P}$ to accelerate the training process.
    Therefore, the loss function for preserving the second-order proximity can be defined as follows.
    \begin{equation} \label{eq_l2}
        \mathcal{L}_2 = \inf_{Q(Z|X) \in Q} \mathbb{E}_{P_X} \mathbb{E}_{Q(Z|X)} \big[ \Vert X {\circ} (X - G(Z)) \Vert^2_2 \big],
    \end{equation}
    where ${\circ}$ denotes the element-wise multiplication.
    The transition matrix $\mathbf{P}$ is used as the input feature X in DVNE.
    The second-order proximity is then preserved by the reconstruction process through forcing nodes with similar neighborhoods to have similar latent representations.

    For purpose of simultaneously preserving the first-order proximity and second-order proximity of networks, DVNE jointly minimizes the loss function of Eq.\eqref{eq_l1} and Eq.\eqref{eq_l2} by combining them together:
    \begin{equation} \label{eq_l12}
        \mathcal{L} = \mathcal{L}_1 + \alpha \mathcal{L}_2.
    \end{equation}

\subsubsection{Optimization}
    Optimizing objective function \eqref{eq_l1} in large graphs is computationally expensive, which needs to compute all the valid triplets in $\mathbf{D}$.
    Hence, we uniformly sample triplets from $\mathbf{D}$ , replacing $\sum_{(i,j,k) \in \mathbf{D}}$ with $\mathbb{E}_{(i,j,k) \sim \mathbf{D}}$ in Eq.\eqref{eq_l1}.
    $M$ triplets are sampled in each iteration from $\mathbf{D}$ to compute the estimation of gradient.

    $Z$ in objective function \eqref{eq_l2} is sampled from $Q(Z|X)$, which is a non-continuous operation without gradient.
    Similar to Variational Auto-Encoders (VAE) \cite{doersch2016tutorial}, the "reparameterization trick" is used here for the optimization of the above objective function via the deep neural networks.
    Firstly, we sample $\epsilon \sim \mathcal{N}(0,\mathbf{I})$. Then, we can calculate $Z = \mu(X) + \Sigma^{1/2}(X) * \epsilon$.
    Consequently, the objective function \eqref{eq_l2} becomes deterministic and continuous in the parameter spaces of encoders $Q$ and decoders $G$, given a fixed $X$ and $\epsilon$, 
    which means the gradients can be computed by backpropagation in deep neural networks. 

\subsubsection{Complexity analysis}
    Algorithm \ref{alg1} lists each step of DVNE.
    The complexity of gradient computation and parameters updating during the training procedure is $O(T * M * (d_{ave}S+SL+L))$, where $T$ is the number of iterations, $M$ stands for the number of edges, $d_{ave}$ represents the average degree of all nodes, $L$ stands for the dimension of embedding vectors and $S$ represents the size of hidden layer.
    Because only non-zero elements in $x_i$ are reconstructed in DVNE, the computational complexity of the first and last hidden layers is $O(d_{ave}S)$, while that of other hidden layers is $O(SL)$. In addition, computaion of the $W_2$ distance among distributions takes $O(L)$.
    In experiments, convergence can be achieved with a small number of iterations $T$ (e.g., $T \leq 50$).

\begin{algorithm}[t]\caption{DVNE Embedding} \label{alg1}
\begin{algorithmic}[1]
    \Require The network $\mathbf{G}=\{\mathbf{V},\mathbf{E}\}$ with the transition matrix $\mathbf{P}$, the parameter $\alpha$
    \Ensure Network embeddings $\{\mathbf{h}_i\}_{i=1}^N$ and updated parameters $\theta = \{ \mathbf{W}^{(i)},\mathbf{b}^{(i)}\}_{i=1}^5$
    \State Initialize parameters $\theta$ by xavier initialization
    \While {$\mathcal{L}$ does not converge}
        \State Uniformly sample $M$ triplets from $\mathbf{D}$
        \State Split these triplets into a number of batches
        \State Compute partial derivative $\partial{\mathcal{L}} / \partial{\theta}$
                with backpropagation algorithm to update ${\theta}$
    \EndWhile
\end{algorithmic}
\end{algorithm}
\section{Dynamic-aware Network Embedding}

\label{sec:DepthLGP_introduction}

Despite the commendable success network embedding has achieved in tasks such as classification and recommendation, most existing algorithms in the literature to date are primarily designed for static networks, where all nodes are known before learning.
However, for large-scale networks, it is infeasible to rerun network embedding whenever new nodes arrive, especially considering the fact that rerunning network embedding also results in the need of retraining all downstream classifiers.
How to efficiently infer proper embeddings for out-of-sample nodes, i.e.,\ nodes that arrive after the embedding process, remains largely unanswered.

Several graph-based methods in the literature can be adapted to infer out-of-sample embeddings given in-sample ones.
Many of them deduce embeddings of new nodes by performing information propagation~\cite{Zhu02learningfrom}, or optimizing a loss that encourages smoothness between linked nodes~\cite{Zhu2003-SLU,Delalleau05efficientnon}.
There are also methods that aim to learn a function mapping node features (e.g.,\ text attributes,
or rows of the adjacency matrix when attributes are unavailable) to outcomes/embeddings, while imposing a manifold regularizer derived from the graph~\cite{Belkin-MRG}. The embeddings of out-of-sample nodes can then be predicted based on their features by these methods.
Nevertheless, existing methods are facing several challenges.
Firstly, the inferred embeddings of out-of-sample nodes should preserve intricate network properties with embeddings of in-sample nodes.
For example, high-order proximity, among many other properties, is deemed especially essential to be preserved by network embedding~\cite{tang2015line,Cao15-GraRep,ou2016asymmetric}, and thus must be reflected by the inferred embeddings.
Secondly, as downstream applications (e.g.,\ classification) will treat in-sample and out-of-sample nodes equally, the inferred embeddings and in-sample embeddings should possess similar characteristics (e.g.,\ magnitude, mean, variance, etc.), i.e.,\ belong to a homogeneous space,
resulting in the need of a model expressive enough to characterize the embedding space.
Finally, maintaining fast prediction speed is crucial, especially considering the highly dynamic nature of real-world networks.
This final point is even more challenging due to the demand of simultaneously fulfilling the previous two requirements.

\begin{figure}[t]
	\centering
	\includegraphics[width=0.46\textwidth]{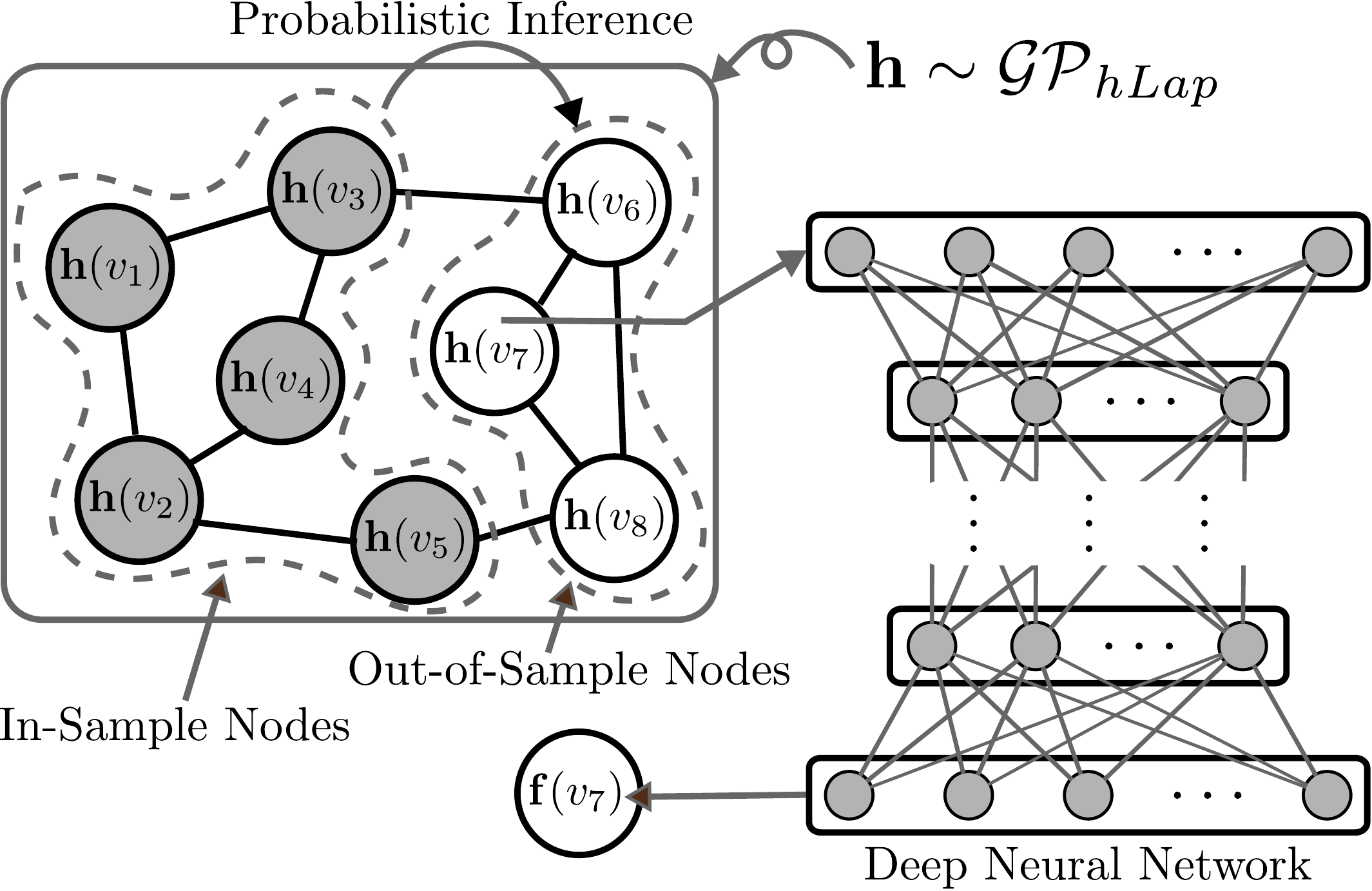}
	\caption{Here $v_6, v_7$, and $v_8$ are out-of-sample nodes. Values of $\mathbf{h}(\cdot)$ are latent states. Values of shaded nodes are learned during training.
		To predict $\mathbf{f}(v_7)$, DepthLGP first predicts $\mathbf{h}(v_7)$ via probabilistic inference, then passes $\mathbf{h}(v_7)$ through a neural network to obtain $\mathbf{f}(v_7)$, figure from~\cite{ma2018depthlgp}.}
	\label{fig-frwk}
\end{figure}

To infer out-of-sample embeddings, Ma~\etal~\cite{ma2018depthlgp} propose a Deeply Transformed High-order Laplacian Gaussian Process (DepthLGP) approach (see Figure~\ref{fig-frwk}) through combining nonparametric probabilistic modeling with deep neural networks.
More specifically, they
 first design a high-order Laplacian Gaussian process (hLGP) prior with a carefully constructed kernel that encodes important network properties such as high-order proximity.
Each node is associated with a latent state that follows the hLGP prior.
They then employ a deep neural network to learn a nonlinear transformation function from these latent states to node embeddings.
The introduction of a deep neural network increases the expressive power of our model and improves the homogeneity of inferred embeddings with in-sample embeddings.
Theories on the expressive power of DepthLGP are derived.
Overall, their proposed DepthLGP model is fast and scalable, requiring
zero knowledge of out-of-sample nodes during training process.
The prediction routine revisits the evolved network rapidly and can produce inference results analytically with desirable time complexity linear with the number of in-sample nodes.
DepthLGP is a general solution, in that it is applicable to embeddings learned by any network embedding algorithms.

\subsection{The DepthLGP Model}

In this section, we first formally formulate the out-of-sample node problem, and then discuss the DepthLGP model as well as theories on its expressive power.

\subsubsection{Problem Definition}
DepthLGP primarily considers undirected networks.
Let $\mathcal{G}$ be the set of all possible networks and $\mathcal{V}$ be the set of all possible nodes.
Given a specific network $G=(V, E) \in \mathcal{G}$ with nodes $V=\{v_1, v_2, \ldots, v_n\} \subset \mathcal{V}$ and edges $E$, a network embedding algorithm aims to learn values of a function $\mathbf{f}:\mathcal{V}\to\mathbb{R}^d$ for nodes in $V$.
As the network evolves over time, a batch of $m$ new nodes $V^{*} = \{v_{n+1}, v_{n+2},\ldots$, $v_{n+m}\} \subset \mathcal{V}\setminus V$ arrives, and expands $G$ into a larger network $G'=(V', E')$, where $V' = V\cup V^{*}$. Nodes in $V^*$ are called out-of-sample nodes. The problem, then, is to infer values of $\mathbf{f}(v)$ for $v\in V^{*}$, given $G'=(V', E')$ and $\mathbf{f}(v)$ for $v\in V$.

\subsubsection{Model Description}

DepthLGP first assumes that there exists a latent function $\mathbf{h}:\mathcal{V}\to\mathbb{R}^s$,
and the embedding function $\mathbf{f}:\mathcal{V}\to\mathbb{R}^d$ is transformed from the said latent function.
To be more specific, let $\mathbf{g}:\mathbb{R}^s\to\mathbb{R}^d$ be the transformation function.
DepthLGP then assumes that $\mathbf{f}(v) = \mathbf{g}(\mathbf{h}(v))$ for all $v\in\mathcal{V}$.
Since the transformation can potentially be highly nonlinear, the authors use a deep neural network to serve as $\mathbf{g}(\cdot)$.
DepthLGP further assumes that the $s$ output dimensions of $\mathbf{h}(\cdot)$, i.e.\ $h_k:\mathcal{V}\to \mathbb{R}$ for $k =1,2,\ldots, s$, can be modeled independently. In other words, DepthLGP deals with each $h_k(v)$ of $\mathbf{h}(v) = \begin{bmatrix} h_1(v), h_2(v), \ldots, h_s(v) \end{bmatrix}^\top$ separately.

Let us focus on $h_k(\cdot)$ for the moment.
Each $h_k(\cdot)$ is associated with a kernel that measures similarity between nodes of a network.
Take $G'=(V', E')$ with $V'=\{v_1,v_2,\ldots,v_{n+m}\}$ for example, the said kernel produces a kernel matrix $\mathbf{K}_k \in \mathbb{R}^{(n+m)\times (n+m)}$ for $G'$:
\begin{align*}
    \mathbf{K}_k & \triangleq \left[\mathbf{I} + \eta_k \mathbf{L}(\mathbf{\hat{A}}_k) +
        \zeta_k \mathbf{L}(\mathbf{\hat{A}}_k\mathbf{\hat{A}}_k) \right]^{-1},\\
    \mathbf{\hat{A}}_k & \triangleq \mathrm{diag}(\boldsymbol{\alpha}_k) \mathbf{A}'
        \mathrm{diag}(\boldsymbol{\alpha}_k), \\
    \boldsymbol{\alpha}_k & \triangleq [a_{v_1}^{(k)}, a_{v_2}^{(k)}, \ldots, a_{v_{n+m}}^{(k)}]^\top,
\end{align*}
where $\mathbf{A}'$ is the adjacency matrix of $G'$,
while $\eta_k \in [0, \infty), \zeta_k \in [0, \infty)$ and $a_{v}^{(k)}\in [0,1]$ for $v \in \mathcal{V}$ are parameters of the kernel. $\mathrm{diag}(\cdot)$ returns a diagonal matrix corresponding to its vector input, while $\mathbf{L}(\cdot)$ treats its input as an adjacency matrix and returns the corresponding Laplacian matrix, i.e.,\ $\mathbf{L}(\mathbf{A}) = \mathrm{diag}(\sum_i \mathbf{A}_{:, i}) - \mathbf{A}$.

The parameters of the proposed kernel have clear physical meanings.
$\eta_k$ indicates the strength of first-order proximity (i.e.,\ connected nodes are likely to be similar), while $\zeta_k$ is for second-order proximity (i.e.,\ nodes with common neighbors are likely to be similar).
On the other hand,
$a^{(k)}_v$ represents a node weight, i.e.,\ how much attention we should pay to node $v$ when conducting prediction.
Values of $a^{(k)}_v$ for in-sample nodes ($v\in V$) are learned along with $\eta_k$ and $\zeta_k$ (as well as parameters of the neural network $\mathbf{g}(\cdot)$) during training,
while values of $a^{(k)}_v$ for out-of-sample nodes ($v\in V^*$) are set to 1 during prediction, since we are always interested in these new nodes when inferring embeddings for them.
Node weights help DepthLGP avoid uninformative nodes. For example, in a social network, this design alleviates harmful effects of ``bot'' users that follow a large amount of random people and spam uninformative contents.

It is easy to see that $\mathbf{K}_k$ is positive definite, hence a valid kernel matrix.
The kernel in DepthLGP can be seen as a generalization of the regularized Laplacian kernel~\cite{Smola03-reglap}, in that the authors further introduce node weighting and a second-order term.
This kernel is referred as the high-order Laplacian kernel.

DepthLGP assumes that each sub-function $h_k(\cdot)$ follows a zero mean Gaussian process (GP)~\cite{Rasmussen05-GPML} parameterized by the high-order Laplacian kernel, i.e.,\
$h_k \sim \mathcal{GP}_{hLap}^{(k)}$.
This is equivalent to say that: For any $G_t = (V_t, E_t) \in \mathcal{G}$ with $V_t = \{v^{(t)}_1, v^{(t)}_2, \ldots, v^{(t)}_{n_t}\} \subset \mathcal{V}$, we have
\begin{align*}
    [h_k(v^{(t)}_1), h_k(v^{(t)}_2), \ldots, h_k(v^{(t)}_{n_t})]^\top
    \sim \mathcal{N}(\mathbf{0}, \mathbf{K}_k^{(t)}),
\end{align*}
where $\mathbf{K}_k^{(t)}$ is the corresponding high-order Laplacian kernel matrix computed on $G_t$.

The DepthLGP model can be summarized as follows.
\begin{align*}
    h_k  &\sim \mathcal{GP}_{hLap}^{(k)},&k=1,2,\ldots, s,\\
    \mathbf{h}(v)  &\triangleq [h_1(v), h_2(v),\ldots, h_s(v)]^\top, &v\in\mathcal{V},\\
    \mathbf{f}(v) \mid \mathbf{h}(v) &\sim
        \mathcal{N}(\mathbf{g}(\mathbf{h}(v)), \sigma^2 \mathbf{I}),&v\in \mathcal{V}.
\end{align*}
where $\sigma$ is a hyper-parameter to be manually specified.
The neural network $\mathbf{g}(\cdot)$ is necessary here,
since $\mathbf{f}(\cdot)$ itself might not follow the GP prior eactly.
The introduction of $\mathbf{g}(\cdot)$ allows the model to fit $\mathbf{f}(\cdot)$ more accurately.

\subsubsection{Prediction}

\begin{algorithm}[!t]
    \begin{algorithmic}[1]
        \Require $G'=(V', E')$ \Comment{$G=(V,E)$ evolves into $G'$.}
        \Ensure predicted values of $\mathbf{f}(v)$,
            $v\in V^*$ \Comment{$V^* \triangleq V'\setminus V$.}
        \LineComment{Let $V=\{v_1, v_2, \ldots, v_n\}$ be old nodes.}
        \LineComment{Let $V^*=\{v_{n+1}, v_{n+2}, \ldots, v_{n+m}\}$ be new nodes.}
        \LineComment{Let $\mathbf{A}'$ be the adjacency matrix of $G'$.}
        \For{$k=1,2,\ldots, s$}
            \LineComment{Values of $a_v^{(k)}$ are set to 1 for $v\in V^*$.}
            \State $\boldsymbol{\alpha} \gets [a_{v_1}^{(k)}, a_{v_2}^{(k)},
                \ldots, a_{v_{n+m}}^{(k)}]^\top$
            \State $\hat{\mathbf{A}}\gets
                \mathrm{diag}(\boldsymbol{\alpha})\;\mathbf{A}\;\mathrm{diag}(\boldsymbol{\alpha})$
            \LineComment{Function $\mathbf{L}(\cdot)$ below treats $\hat{\mathbf{A}}$ and $\hat{\mathbf{A}}\hat{\mathbf{A}}$ as adjacency matrices, and returns their Laplacian matrices.}
            \State $\mathbf{M} \gets \mathbf{I} +
                \eta_k \mathbf{L}(\hat{\mathbf{A}}) +
                \zeta_k \mathbf{L}(\hat{\mathbf{A}}\hat{\mathbf{A}})$
            \State $\mathbf{M}_{*,*}\gets$ the bottom-right $m\times m$ block of $\mathbf{M}$
            \State $\mathbf{M}_{*,x}\gets$ the bottom-left $m\times n$ block of $\mathbf{M}$
            \LineComment{Let $\mathbf{z}_x^{(k)} \triangleq
                [h_k(v_1), h_k(v_2), \ldots, h_k(v_n)]^\top$.}
            \LineComment{Compute ${\mathbf{M}_{*,x}}\mathbf{z}^{(k)}_x$ first below for efficiency.}
            \State $\mathbf{z}^{(k)}_* \gets
                -\mathbf{M}_{*,*}^{-1} {\mathbf{M}_{*,x}}\mathbf{z}^{(k)}_x$
            \Comment{$\mathbf{z}_*^{(k)}$ is a prediction of
                $[h_k(v_{n+1}), h_k(v_{n+2}),$ $\ldots,$ $h_k(v_{n+m})]^\top$.}
        \EndFor
        \For{$v \in V^*$}
            \LineComment{Previous lines have produced a prediction of $\mathbf{h}(v) = [h_1(v), h_2(v), \ldots, h_s(v)]^\top$. The line below now uses the said prediction to further predict $\mathbf{f}(v)$.}
            \State compute $\mathbf{g}(\mathbf{h}(v))$
                \Comment{It is a prediction of $\mathbf{f}(v)$.}
        \EndFor
    \end{algorithmic}
    \caption{DepthLGP's Prediction Routine}
    \label{alg-predict}
\end{algorithm}

Before new nodes arrive, we have the initial network $G=(V, E)$ with $V=\{v_1, v_2, \ldots, v_n\}$, and know the values of $\mathbf{f}(v)$ for $v\in V$.
The prediction routine assumes that there is a training procedure (see subsection ``Training'') conducted on $G$ and $\mathbf{f}(v)$ for $v\in V$ before new nodes arrive,
and the training procedure learns $\eta_k, \zeta_k, h_k(v), a^{(k)}_v$ for $k=1,2,\ldots, s$ and $v\in V$, as well as parameters of the transformation function $\mathbf{g}(\cdot)$.

As the network evolves over time, $m$ new nodes $V^{*} = \{v_{n+1}, v_{n+2},\ldots$, $v_{n+m}\}$ arrive and $G$ evolves into $G'=(V', E')$ with $V'=V\cup V^{*}$.
DepthLGP's prediction routine aims to predict $\mathbf{f}(v)$ for $v\in V^{*}$ by maximizing
$ 
p(\{\mathbf{f}(v): v\in V^{*} \} \mid \{\mathbf{f}(v): v\in V\}, \{\mathbf{h}(v): v\in V\}),
$ 
which, according to our model, is equal to
\begin{align*}
p(\{\mathbf{f}(v): v\in V^{*} \} \mid \{\mathbf{h}(v): v\in V\}).
\end{align*}
Yet, it requires integrating over all possible $\mathbf{h}(v)$ for $v\in V^{*}$.
DepthLGP therefore approximates it by maximizing
\begin{align*}
    p(\{\mathbf{f}(v): v\in V^{*} \}, \{\mathbf{h}(v): v\in V^{*} \} \mid \{\mathbf{h}(v): v\in V\}),
\end{align*}
which is equal to
\begin{align*}
    & p(\{\mathbf{f}(v): v\in V^{*} \} \mid  \{\mathbf{h}(v): v\in V^{*} \}) \\
    &\quad\; \times p(\{\mathbf{h}(v): v\in V^{*} \} \mid \{\mathbf{h}(v): v\in V\}).
\end{align*}
It can be maximized~\footnote{
Note that the first term $p(\{\mathbf{f}(v): v\in V^{*} \} \mid  \{\mathbf{h}(v): v\in V^{*} \})$ is maximized with $\mathbf{f}(v) = \mathbf{g}(\mathbf{h}(v))$, and the maximum value of this probability density is a \emph{constant} unrelated with $\mathbf{h}(v)$. Hence we can focus on maximizing the second term first.}
by first maximizing the second term, i.e.\ $p(\mathbf{h}(v): v\in V^{*} \} \mid \{\mathbf{h}(v): v\in V\})=\prod_{k=1}^s p(\{h_k(v): v\in V^{*} \} \mid \{h_k(v): v\in V\})$, and then setting $\mathbf{f}(v) = \mathbf{g}(\mathbf{h}(v))$ for $v\in V^*$.

Let us now focus on the subproblem, i.e.,\ maximizing
\begin{align} \label{eq-subgoal}
    p(\{h_k(v): v\in V^{*} \} \mid \{h_k(v): v\in V\}).
\end{align}
Since $h_k\sim \mathcal{GP}_{hLap}$, by definition we have
\begin{align*}
    [h_k(v_1), h_k(v_2), \ldots, &h_k(v_n),  h_k(v_{n+1}), \ldots, h_k(v_{n+m})]^\top \\
    & \sim \mathcal{N}(\mathbf{0}, \mathbf{K}_k),
\end{align*}
where $\mathbf{K}_k$ is the corresponding kernel matrix computed on $G'$.
We then have the following result:
\begin{align*}
    \mathbf{z}_*^{(k)} \mid \mathbf{z}_x^{(k)} & \sim \mathcal{N}
        (\mathbf{K}_{*,x} \mathbf{K}_{x,x}^{-1} \mathbf{z}_x^{(k)},
        {\mathbf{K}_{*,*}}- \mathbf{K}_{*,x}\mathbf{K}_{x, x}^{-1}\mathbf{K}_{*,x}^\top),\\
    \mathbf{z}_x^{(k)} &\triangleq \begin{bmatrix}h_k(v_1), h_k(v_2), \ldots, h_k(v_n)\end{bmatrix}^\top, \\
    \mathbf{z}_*^{(k)} &\triangleq \begin{bmatrix}h_k(v_{n+1}), h_k(v_{n+2}), \ldots, h_k(v_{n+m})\end{bmatrix}^\top,
\end{align*}
where
$\mathbf{K}_{x,x}$, $\mathbf{K}_{*,x}$, and $\mathbf{K}_{*,*}$ are respectively the
top-left $n\times n$, bottom-left $m\times n$, and bottom-right $m\times m$ blocks of $\mathbf{K}_k$.
Though $\mathbf{K}_{*,x} \mathbf{K}_{x,x}^{-1} \mathbf{z}_x^{(k)}$ is expensive to compute, it can thankfully be proved to be equivalent to $-\mathbf{M}_{*,*}^{-1} {\mathbf{M}_{*,x}}\mathbf{z}^{(k)}_x$,
where $\mathbf{M}_{*,x}$ and $\mathbf{M}_{*,*}$ are respectively the bottom-left $m\times n$ and bottom-right $m\times m$ blocks of $\mathbf{K}_k^{-1}$. $\mathbf{K}_k^{-1}$ is cheap to obtain as the matrix inversion gets cancelled out. And computing $\mathbf{M}_{*,*}^{-1}$ is fast, since $m\ll n$.
As a result, Eq.~\eqref{eq-subgoal} is maximized as:
\begin{align*}
    \mathbf{z}_*^{(k)}=-\mathbf{M}_{*,*}^{-1} {\mathbf{M}_{*,x}}\mathbf{z}^{(k)}_x.
\end{align*}

As a side node, maximizing Eq.~\eqref{eq-subgoal} is in fact equivalent to minimizing the following criterion:
\begin{align*}
                         & \sum_{u\in V'} \left[h_k(u)\right]^2 + \\
    \frac{1}{2}\eta_k  & \sum_{u,v\in V'} a^{(k)}_u A'_{uv} a^{(k)}_v \left[h_k(u) - h_k(v)\right]^2 + \\
    \frac{1}{2}\zeta_k & \sum_{u,v,w\in V'}
               a^{(k)}_u A'_{uw} a^{(k)}_w a^{(k)}_w A'_{wv} a^{(k)}_v
               \left[h_k(u) - h_k(v)\right]^2,
\end{align*}
where $A'_{uv}$ is the edge weight (zero if not connected) between $u$ and $v$ in $G'$.
This form hints at the physical meanings of $\eta, \zeta$ and $a^{(k)}_v$ from another perspective.

The prediction routine is summarized in Algorithm~\ref{alg-predict}.

\subsubsection{Training}

Training is conducted on the initial network, i.e.,\ $G=(V, E)$, with the values of $\mathbf{f}(v)$ for $v\in V$. Since it does not depend on the evolved network $G'=(V', E')$, it can be carried out before new nodes arrive. It aims to find suitable parameters of the neural network $\mathbf{g}(\cdot)$ and proper values of $\eta_k, \zeta_k, a^{(k)}_v, h_k(v)$ for $v \in V$ and $k=1,2,\ldots, s$.

The authors apply empirical minimum risk (ERM) training to DepthLGP model.
ERM training of a probabilistic model, though not as conventional as maximum likelihood estimation (MLE) and maximum a posteriori (MAP) estimation, has been explored by many researchers before, e.g.,\ \cite{pmlr-v15-stoyanov11a}.
Using ERM training here eliminates the need to specify $\sigma$, and is faster and more scalable as it avoids computing determinants.

The training procedure is listed in Algorithm~\ref{alg-train}.
The basic idea is to first sample some subgraphs from $G$, then treat a small portion of nodes in each subgraph as if they were out-of-sample nodes, and minimize empirical risk on these training samples (i.e.,\ minimize mean squared error of predicted embeddings).

Now let us describe how each training sample, $G_t'=(V_t', E_t')$, is sampled.
DepthLGP first samples a subset of nodes, $V_t^*$, from $G$, along a random walk path.
Nodes in $V_t^*$ are treated as ``new'' nodes. 
DepthLGP then samples a set of nodes, $V_t$, from the neighborhood of $V_t^*$.
DepthLGP defines the neighborhood of $V_t^*$ to be nodes that are no more than two steps away from $V_t^*$. 
Finally, let $V_t' = V_t^*\cup V_t$, and $G_t'$ be the subgraph induced in $G$ by $V_t'$.

Optimization is be done with a gradient-based method and the authors use Adam~\cite{kingma2014adam} for this purpose. Gradients are computed using back-propagation~\cite{Rumelhart-BP,Dreyfus196230}.
Good parameter initialization can substantially improve convergence speed. To allow easy initialization, they use a residual network (more strictly speaking, a residual block)~\cite{He2015} to serve as $\mathbf{g}(\cdot)$. In other words, DepthLGP chooses $\mathbf{g}(\cdot)$ to be of the form $\mathbf{g}(\mathbf{x}) = \mathbf{x} + \mathbf{\tilde{g}}(\mathbf{x})$, where $\mathbf{\tilde{g}}(\cdot)$ is a feed-forward neural network. In this case, $s=d$. Thus it is able to initialize values of $\mathbf{h}(v)$ to be values of $\mathbf{f}(v)$ for nodes in $V$.

\begin{algorithm}[!t]
    \begin{algorithmic}[1]
        \Require $G=(V,E)$;\; $\mathbf{f}(v)$ for $v\in V$
        \Ensure $\eta_k, \zeta_k, a^{(k)}_v, h_k(v)$ for $v\in V$ and $k=1,2,\ldots, s$;\;
            parameters of the neural network $\mathbf{g}(\cdot)$
        \For{$t=1,2,\ldots, T$}
            \LineComment{See subsection ``Training'' for more details on how to sample $V_t^*$ and $V_t$.}
            \State $V_t^* \gets $ a few nodes sampled along a random walk
            \State $V_t \gets$ some nodes in $V_t^*$'s neighborhood
            \State $V_t' \gets V_t \cup V_t^*$
            \State $G_t'\gets$ the subgraph induced in $G$ by $V_t'$
            \State Execute Algorithm~\ref{alg-predict}, but using $G_t'$ in place of $G'$,
            $V_t$ in place of old nodes, and $V_t^*$ in place of new nodes.
            Save its prediction of $\mathbf{f}(v)$ as $\mathbf{\tilde{f}}(v)$ for $v\in V_t^*$.
            \State loss $\gets \frac{1}{|V_t^*|}
                \sum_{v\in V_t^*}\|\mathbf{f}(v)-\mathbf{\tilde{f}}(v)\|^2_{l^2}$
            \State Use back-propagation to compute the gradient of the loss with respect to
                $\eta_k, \zeta_k, a^{(k)}_v, h_k(v)$ for $v\in V_t$
                and parameters of $\mathbf{g}(\cdot)$.
            \State Apply gradient descent.
        \EndFor
    \end{algorithmic}
    \caption{DepthLGP's Training Routine}
    \label{alg-train}
\end{algorithm}

\subsubsection{On the Expressive Power of DepthLGP}

Theorem~\ref{theo-expr} below demonstrates the expressive power of DepthLGP, i.e.,\ to what degree it can model arbitrary $\mathbf{f}:\mathcal{V}\to\mathbb{R}^d$.
\begin{theorem}[Expressive Power]\label{theo-expr}
    For any $\epsilon > 0$, any nontrivial $G=(V,E)$ and any $\mathbf{f}:\mathcal{V}\to\mathbb{R}^d$, there exists a parameter setting for DepthLGP, such that: For any $v^*\in V$, after deleting all information (except $G$) related with $v^*$, DepthLGP can still recover $\mathbf{f}(v^*)$ with error less than $\epsilon$, by treating $v^*$ as a new node and using Algorithm~\ref{alg-predict} on $G$.
\end{theorem}
\begin{remark}
A nontrivial $G$ means that all connected components of $G$ have at least three nodes.
Information related with $v^*$ includes $\mathbf{f}(v^*), h_k(v^*)$ and $a^{(k)}_{v^*}$ for $k=1,2,\ldots,s$
(note that during prediction, $a^{(k)}_{v^*}$ is replaced by 1 since $v^*$ is treated as a new node).
Error is expressed in terms of $l^2$-norm.
It can be proved with a constructive proof based on the universal approximation property of neural networks~\cite{Cybenko1989,Hornik:1991}.
\end{remark}

Theorem~\ref{theo-2nd} below then emphasizes the importance of second-order proximity: Even though DepthLGP leverages the expressive power of a neural network, modeling second-order proximity is still necessary.

\begin{theorem}[On Second-Order Proximity] \label{theo-2nd}
    Theorem~\ref{theo-expr} will not hold if DepthLGP does not model second-order proximity.
    That is,
    there will exist $G=(V,E)$ and $\mathbf{f}:\mathcal{V}\to\mathbb{R}^d$ that DepthLGP cannot model,
    if $\zeta_k$ is fixed to zero.
\end{theorem}

\subsection{Extensions and Variants}

\subsubsection{Integrating into an Embedding Algorithm}
DepthLGP can also be easily incorporated into an existing network embedding algorithm to derive a new embedding algorithm capable of addressing out-of-sample nodes.
Take node2vec~\cite{Grover16-node2vec} for example.
For an input network $G=(V, E)$ with $V=\{v_1,\ldots, v_n\}$,
node2vec's training objective can be abstracted as
\begin{align*}
    \min_{\theta, \mathbf{F}} \mathcal{L}_\theta(\mathbf{F}, G),
\end{align*}
where columns of $\mathbf{F}\in\mathbb{R}^{d\times n}$ are target node embeddings to be learned, and $\theta$ contains parameters other than $\mathbf{F}$.

Let us use $\mathbf{f}_{\phi}:\mathcal{V}\to\mathbb{R}^d$
to represent a function parameterized by $\phi$.
This function is defined as follows. For $v\in V$, it first samples nodes from $v$'s neighborhood (see subsection ``Training'' on how to sample them) and induces a subgraph from $G$ containing $v$ and these sampled nodes. It then treats $v$ as a new node, nodes sampled from $v$'s neighborhood as old nodes, and runs Algorithm~\ref{alg-predict} on the induced subgraph to obtain a prediction of $v$'s embedding---this is the value of $\mathbf{f}_{\phi}(v)$. By definition, $\phi$ contains parameters of a neural network, $\eta_k, \zeta_k$, $a^{(k)}_v$, and $h_k(v)$ for $v\in V$, $k=1,2,\ldots,s$.
To derive a new embedding algorithm based on node2vec, we can simply change the training objective to:
\begin{align*}
    \min_{\theta, \phi} \mathcal{L}_\theta([\mathbf{f}_\phi(v_1), \mathbf{f}_\phi(v_2), \ldots,
                                            \mathbf{f}_\phi(v_n)], G).
\end{align*}
The authors name this new algorithm node2vec++, where $\mathbf{f}_\phi(v)$ is node $v$'s embedding.
Clearly, node2vec++ can handle out-of-sample nodes efficiently in the same fashion as DepthLGP.

\subsubsection{Efficient Variants}

When predicting node embeddings for out-of-sample nodes, DepthLGP can collectively infer all of them in one pass. However, if the number of newly arrived nodes is large, it is more efficient (and more memory-saving) to process new nodes in a batch-by-batch way: For each unprocessed new node $v$, find the largest connected component containing $v$ and other new nodes (but not old nodes). Let $V^*_t$ be nodes in the connected component, and $V_t$ be old nodes sampled from $V^*_t$'s neighborhood (see subsection ``Training'' on how to sample them). Then it is possible to run Algorithm~\ref{alg-predict} on the subgraph induced by $V^*_t\cup V_t$ to obtain prediction for new nodes in $V_t^*$. Repeat this process until all new nodes are processed.

Some simplifications can be made to DepthLGP without sacrificing much performance while allowing faster convergence and a more efficient implementation of Algorithm~\ref{alg-predict}. In particular, sharing node weights across different dimensions, i.e.\ keeping $a^{(1)}_v = \ldots = a^{(s)}_v$, hurts little for most mainstream embedding algorithms (though theoretically it will reduce the expressive power of DepthLGP).
Similarly, for node2vec and DeepWalk, we can keep $\eta_1=\ldots=\eta_s$ and $\zeta_1=\ldots=\zeta_s$,
since they treat different dimensions of node embeddings equally. For LINE, however, it is better to keep $\eta_1=\ldots=\eta_{\frac{s}{2}}$ ($\zeta_1=\ldots=\zeta_{\frac{s}{2}}$)
and $\eta_{\frac{s}{2}+1}=\ldots=\eta_s$ ($\zeta_{\frac{s}{2}+1}=\ldots=\zeta_s$) separately,
because an embedding produced by LINE is the result of concatenating two sub-embeddings (for 1st- and 2nd-order proximity respectively).

\section{Conclusion and Future Work}
This chapter introduces the problem of graph representation/network embedding and the challenges lying
in the literature, i.e., high non-linearity, structure-preserving, property-preserving and sparsity. 
Given its success in handling large-scale non-linear
data in the past decade, we remark that deep neural network (i.e., deep learning)
serves as an adequate candidate to tackle these challenges and highlight the promising potential for combining graph representation/network embedding with deep neural network.
We select five representative models on deep network embedding/graph representation for discussions, i.e., structural deep network embedding (SDNE), 
deep variational network embedding (DVNE), deep recursive network embedding (DRNE),
deeply transformed high-order Laplacian Gaussian process (DepthLGP) based network embedding
and deep hyper-network embedding (DHNE).
In particular, SDNE and DRNE focus on structure-aware network embedding, which preserve the high order proximity and global structure respectively.
By extending vertex pairs to vertex tuples, DHNE targets at learning embeddings for vertexes with various types
in heterogeneous hyper-graphs and preserving the corresponding hyper structures.
DVNE focuses on the uncertainties in graph representations and DepthLGP aims to learn accurate embeddings for new nodes in dynamic networks.
Our discussions center around two aspects in graph representation/network embedding i) deep structure-oriented network embedding and ii) deep property-oriented network embedding. We hope that readers may benefit from our discussions.

The above discussions of the state-of-the-art network embedding algorithms highly demonstrate that the research field of network embedding is still young and promising. 
Selecting appropriate methods is a crucial question for tackling practical applications with network embedding.
The foundation here is the property and structure preserving issue. 
Serious information in the embedding space may lost if the important network properties cannot be retained and the network structure cannot be well preserved, which damages the analysis in the sequel. 
The off-the-shelf machine learning methods can be applied based on the property and structure preserving network embedding. 
Available side information can be fed into network embedding. 
Moreover, for some certain applications, the domain knowledge can also be introduced as advanced information. 
At last, existing network embedding methods are mostly designed for static networks,
while not surprisingly, many networks in real world applications are evolving over time. 
Therefore, novel network embedding methods to deal with the dynamic nature of evolving networks are highly desirable. 

\section*{Acknowledgment}
We thank Ke Tu (DRNE and DHNE), Daixin Wang (SDNE), Dingyuan Zhu (DVNE) and Jianxin Ma (DepthLGP) for providing us with valuable materials.

Xin Wang is the corresponding author.
This work is supported by China Postdoctoral Science Foundation No. BX201700136, 
National Natural Science Foundation of China Major Project No. U1611461 and National Program on Key Basic Research Project No. 2015CB352300.

\small
\bibliographystyle{abbrv}
\bibliography{reference}

\begin{thebibliography}{10}

\bibitem{agarwal2006higher}
S.~Agarwal, K.~Branson, and S.~Belongie.
\newblock Higher order learning with graphs.
\newblock In {\em Proceedings of the 23rd international conference on Machine
  learning}, pages 17--24. ACM, 2006.

\bibitem{ambrosio2008gradient}
L.~Ambrosio, N.~Gigli, and G.~Savar{\'e}.
\newblock {\em Gradient flows: in metric spaces and in the space of probability
  measures}.
\newblock Springer Science \&amp; Business Media, 2008.

\bibitem{ba2016layer}
J.~L. Ba, J.~R. Kiros, and G.~E. Hinton.
\newblock Layer normalization.
\newblock {\em arXiv preprint arXiv:1607.06450}, 2016.

\bibitem{belkin2003laplacian}
M.~Belkin and P.~Niyogi.
\newblock Laplacian eigenmaps for dimensionality reduction and data
  representation.
\newblock {\em Neural computation}, 15(6):1373--1396, 2003.

\bibitem{Belkin-MRG}
M.~Belkin, P.~Niyogi, and V.~Sindhwani.
\newblock Manifold regularization: A geometric framework for learning from
  labeled and unlabeled examples.
\newblock {\em J. Mach. Learn. Res.}, 7:2399--2434, Dec. 2006.

\bibitem{bengio2009learning}
Y.~Bengio.
\newblock Learning deep architectures for ai.
\newblock {\em Foundations and trends{\textregistered} in Machine Learning},
  2(1):1--127, 2009.

\bibitem{bonacich2007some}
P.~Bonacich.
\newblock Some unique properties of eigenvector centrality.
\newblock {\em Social networks}, 29(4):555--564, 2007.

\bibitem{bonneel2015sliced}
N.~Bonneel, J.~Rabin, G.~Peyr{\'e}, and H.~Pfister.
\newblock Sliced and radon wasserstein barycenters of measures.
\newblock {\em Journal of Mathematical Imaging and Vision}, 51(1):22--45, 2015.

\bibitem{bonneel2011displacement}
N.~Bonneel, M.~Van De~Panne, S.~Paris, and W.~Heidrich.
\newblock Displacement interpolation using lagrangian mass transport.
\newblock In {\em ACM Transactions on Graphics (TOG)}, volume~30, page 158.
  ACM, 2011.

\bibitem{bryant1985metric}
V.~Bryant.
\newblock {\em Metric spaces: iteration and application}.
\newblock Cambridge University Press, 1985.

\bibitem{Cao15-GraRep}
S.~Cao, W.~Lu, and Q.~Xu.
\newblock Grarep: Learning graph representations with global structural
  information.
\newblock CIKM '15, pages 891--900, New York, NY, USA, 2015. ACM.

\bibitem{chen2015fast}
C.~Chen and H.~Tong.
\newblock Fast eigen-functions tracking on dynamic graphs.
\newblock In {\em Proceedings of the 2015 SIAM International Conference on Data
  Mining}, pages 559--567. SIAM, 2015.

\bibitem{clement2008elementary}
P.~Clement and W.~Desch.
\newblock An elementary proof of the triangle inequality for the wasserstein
  metric.
\newblock {\em Proceedings of the American Mathematical Society},
  136(1):333--339, 2008.

\bibitem{courty2017learning}
N.~Courty, R.~Flamary, and M.~Ducoffe.
\newblock Learning wasserstein embeddings.
\newblock {\em arXiv preprint arXiv:1710.07457}, 2017.

\bibitem{courty2017optimal}
N.~Courty, R.~Flamary, D.~Tuia, and A.~Rakotomamonjy.
\newblock Optimal transport for domain adaptation.
\newblock {\em IEEE transactions on pattern analysis and machine intelligence},
  39(9):1853--1865, 2017.

\bibitem{cuturi2014fast}
M.~Cuturi and A.~Doucet.
\newblock Fast computation of wasserstein barycenters.
\newblock In {\em International Conference on Machine Learning}, pages
  685--693, 2014.

\bibitem{Cybenko1989}
G.~Cybenko.
\newblock Approximation by superpositions of a sigmoidal function.
\newblock {\em Mathematics of Control, Signals and Systems}, 2(4):303--314, Dec
  1989.

\bibitem{dash2008context}
N.~S. Dash.
\newblock Context and contextual word meaning.
\newblock {\em SKASE Journal of Theoretical Linguistics}, 5(2):21--31, 2008.

\bibitem{de2012blue}
F.~De~Goes, K.~Breeden, V.~Ostromoukhov, and M.~Desbrun.
\newblock Blue noise through optimal transport.
\newblock {\em ACM Transactions on Graphics (TOG)}, 31(6):171, 2012.

\bibitem{Delalleau05efficientnon}
O.~Delalleau, Y.~Bengio, and N.~L. Roux.
\newblock Efficient non-parametric function induction in semi-supervised
  learning.
\newblock In {\em AISTATS '05}, pages 96--103, 2005.

\bibitem{doersch2016tutorial}
C.~Doersch.
\newblock Tutorial on variational autoencoders.
\newblock {\em arXiv preprint arXiv:1606.05908}, 2016.

\bibitem{Dreyfus196230}
S.~Dreyfus.
\newblock The numerical solution of variational problems.
\newblock {\em Journal of Mathematical Analysis and Applications}, 5(1):30 --
  45, 1962.

\bibitem{eom2015tail}
Y.-H. Eom and H.-H. Jo.
\newblock Tail-scope: Using friends to estimate heavy tails of degree
  distributions in large-scale complex networks.
\newblock {\em Scientific reports}, 5, 2015.

\bibitem{erhan2010does}
D.~Erhan, Y.~Bengio, A.~Courville, P.-A. Manzagol, P.~Vincent, and S.~Bengio.
\newblock Why does unsupervised pre-training help deep learning?
\newblock {\em The Journal of Machine Learning Research}, 11:625--660, 2010.

\bibitem{givens1984class}
C.~R. Givens, R.~M. Shortt, et~al.
\newblock A class of wasserstein metrics for probability distributions.
\newblock {\em The Michigan Mathematical Journal}, 31(2):231--240, 1984.

\bibitem{glorot2011deep}
X.~Glorot, A.~Bordes, and Y.~Bengio.
\newblock Deep sparse rectifier neural networks.
\newblock In {\em Proceedings of the Fourteenth International Conference on
  Artificial Intelligence and Statistics}, pages 315--323, 2011.

\bibitem{Grover16-node2vec}
A.~Grover and J.~Leskovec.
\newblock Node2vec: Scalable feature learning for networks.
\newblock KDD '16, pages 855--864, New York, NY, USA, 2016. ACM.

\bibitem{He2015}
K.~He, X.~Zhang, S.~Ren, and J.~Sun.
\newblock Deep residual learning for image recognition.
\newblock {\em arXiv preprint arXiv:1512.03385}, 2015.

\bibitem{hinton2012deep}
G.~Hinton, L.~Deng, D.~Yu, G.~E. Dahl, A.-r. Mohamed, N.~Jaitly, A.~Senior,
  V.~Vanhoucke, P.~Nguyen, T.~N. Sainath, et~al.
\newblock Deep neural networks for acoustic modeling in speech recognition: The
  shared views of four research groups.
\newblock {\em Signal Processing Magazine, IEEE}, 29(6):82--97, 2012.

\bibitem{hinton2006fast}
G.~E. Hinton, S.~Osindero, and Y.-W. Teh.
\newblock A fast learning algorithm for deep belief nets.
\newblock {\em Neural computation}, 18(7):1527--1554, 2006.

\bibitem{hochreiter1997long}
S.~Hochreiter and J.~Schmidhuber.
\newblock Long short-term memory.
\newblock {\em Neural computation}, 9(8):1735--1780, 1997.

\bibitem{holland1972holland}
P.~W. Holland and S.~Leinhardt.
\newblock Holland and leinhardt reply: some evidence on the transitivity of
  positive interpersonal sentiment, 1972.

\bibitem{Hornik:1991}
K.~Hornik.
\newblock Approximation capabilities of multilayer feedforward networks.
\newblock {\em Neural Netw.}, 4(2):251--257, Mar. 1991.

\bibitem{jamali2010matrix}
M.~Jamali and M.~Ester.
\newblock A matrix factorization technique with trust propagation for
  recommendation in social networks.
\newblock In {\em Proceedings of the fourth ACM conference on Recommender
  systems}, pages 135--142. ACM, 2010.

\bibitem{jin2001structure}
E.~M. Jin, M.~Girvan, and M.~E. Newman.
\newblock Structure of growing social networks.
\newblock {\em Physical review E}, 64(4):046132, 2001.

\bibitem{kingma2014adam}
D.~Kingma and J.~Ba.
\newblock Adam: A method for stochastic optimization.
\newblock {\em arXiv preprint arXiv:1412.6980}, 2014.

\bibitem{kolouri2017optimal}
S.~Kolouri, S.~R. Park, M.~Thorpe, D.~Slepcev, and G.~K. Rohde.
\newblock Optimal mass transport: Signal processing and machine-learning
  applications.
\newblock {\em IEEE Signal Processing Magazine}, 34(4):43--59, 2017.

\bibitem{krizhevsky2012imagenet}
A.~Krizhevsky, I.~Sutskever, and G.~E. Hinton.
\newblock Imagenet classification with deep convolutional neural networks.
\newblock In {\em Advances in neural information processing systems}, pages
  1097--1105, 2012.

\bibitem{lecun2015deep}
Y.~LeCun, Y.~Bengio, and G.~Hinton.
\newblock Deep learning.
\newblock {\em Nature}, 521(7553):436--444, 2015.

\bibitem{lecun2006tutorial}
Y.~LeCun, S.~Chopra, R.~Hadsell, M.~Ranzato, and F.~Huang.
\newblock A tutorial on energy-based learning.
\newblock {\em Predicting structured data}, 1(0), 2006.

\bibitem{leicht2006vertex}
E.~A. Leicht, P.~Holme, and M.~E. Newman.
\newblock Vertex similarity in networks.
\newblock {\em Physical Review E}, 73(2):026120, 2006.

\bibitem{liben2007link}
D.~Liben-Nowell and J.~Kleinberg.
\newblock The link-prediction problem for social networks.
\newblock {\em Journal of the American society for information science and
  technology}, 58(7):1019--1031, 2007.

\bibitem{luo2011cauchy}
D.~Luo, F.~Nie, H.~Huang, and C.~H. Ding.
\newblock Cauchy graph embedding.
\newblock In {\em Proceedings of the 28th International Conference on Machine
  Learning (ICML-11)}, pages 553--560, 2011.

\bibitem{ma2018depthlgp}
J.~Ma, P.~Cui, and W.~Zhu.
\newblock Depthlgp: Learning embeddings of out-of-sample nodes in dynamic
  networks.
\newblock pages 370--377. AAAI, 2018.

\bibitem{mikolov2010recurrent}
T.~Mikolov, M.~Karafi{\'a}t, L.~Burget, J.~{\v{C}}ernock{\`y}, and
  S.~Khudanpur.
\newblock Recurrent neural network based language model.
\newblock In {\em Eleventh Annual Conference of the International Speech
  Communication Association}, pages 1045--1048, 2010.

\bibitem{mikolov2013distributed}
T.~Mikolov, I.~Sutskever, K.~Chen, G.~S. Corrado, and J.~Dean.
\newblock Distributed representations of words and phrases and their
  compositionality.
\newblock In {\em Advances in neural information processing systems}, pages
  3111--3119, 2013.

\bibitem{nathan2017dynamic}
E.~Nathan and D.~A. Bader.
\newblock A dynamic algorithm for updating katz centrality in graphs.
\newblock In {\em Proceedings of the 2017 IEEE/ACM International Conference on
  Advances in Social Networks Analysis and Mining 2017}, pages 149--154. ACM,
  2017.

\bibitem{ou2016asymmetric}
M.~Ou, P.~Cui, J.~Pei, Z.~Zhang, and W.~Zhu.
\newblock Asymmetric transitivity preserving graph embedding.
\newblock In {\em Proc. of ACM SIGKDD}, pages 1105--1114, 2016.

\bibitem{page1999pagerank}
L.~Page, S.~Brin, R.~Motwani, and T.~Winograd.
\newblock The pagerank citation ranking: Bringing order to the web.
\newblock Technical report, Stanford InfoLab, 1999.

\bibitem{perozzi2014deepwalk}
B.~Perozzi, R.~Al-Rfou, and S.~Skiena.
\newblock Deepwalk: Online learning of social representations.
\newblock In {\em SIGKDD}, pages 701--710. ACM, 2014.

\bibitem{Rasmussen05-GPML}
C.~E. Rasmussen and C.~K.~I. Williams.
\newblock {\em Gaussian Processes for Machine Learning (Adaptive Computation
  and Machine Learning)}.
\newblock The MIT Press, 2005.

\bibitem{rossi2015role}
R.~A. Rossi and N.~K. Ahmed.
\newblock Role discovery in networks.
\newblock {\em IEEE Transactions on Knowledge and Data Engineering},
  27(4):1112--1131, 2015.

\bibitem{Rumelhart-BP}
D.~E. Rumelhart, G.~E. Hinton, and R.~J. Williams.
\newblock Neurocomputing: Foundations of research.
\newblock chapter Learning Representations by Back-propagating Errors, pages
  696--699. MIT Press, Cambridge, MA, USA, 1988.

\bibitem{salakhutdinov2009semantic}
R.~Salakhutdinov and G.~Hinton.
\newblock Semantic hashing.
\newblock {\em International Journal of Approximate Reasoning}, 50(7):969--978,
  2009.

\bibitem{shaw2009structure}
B.~Shaw and T.~Jebara.
\newblock Structure preserving embedding.
\newblock In {\em Proceedings of the 26th Annual International Conference on
  Machine Learning}, pages 937--944. ACM, 2009.

\bibitem{siegelmann1995computational}
H.~T. Siegelmann and E.~D. Sontag.
\newblock On the computational power of neural nets.
\newblock {\em Journal of computer and system sciences}, 50(1):132--150, 1995.

\bibitem{Smola03-reglap}
A.~J. Smola and R.~Kondor.
\newblock {\em Kernels and Regularization on Graphs}, pages 144--158.
\newblock Springer Berlin Heidelberg, Berlin, Heidelberg, 2003.

\bibitem{socher2013recursive}
R.~Socher, A.~Perelygin, J.~Y. Wu, J.~Chuang, C.~D. Manning, A.~Y. Ng, and
  C.~Potts.
\newblock Recursive deep models for semantic compositionality over a sentiment
  treebank.
\newblock In {\em Proceedings of the conference on empirical methods in natural
  language processing (EMNLP)}, volume 1631, page 1642. Citeseer, 2013.

\bibitem{pmlr-v15-stoyanov11a}
V.~Stoyanov, A.~Ropson, and J.~Eisner.
\newblock Empirical risk minimization of graphical model parameters given
  approximate inference, decoding, and model structure.
\newblock AISTATS'11, Fort Lauderdale, Apr. 2011.

\bibitem{sun2008hypergraph}
L.~Sun, S.~Ji, and J.~Ye.
\newblock Hypergraph spectral learning for multi-label classification.
\newblock In {\em Proceedings of the 14th ACM SIGKDD international conference
  on Knowledge discovery and data mining}, pages 668--676. ACM, 2008.

\bibitem{tang2015line}
J.~Tang, M.~Qu, M.~Wang, M.~Zhang, J.~Yan, and Q.~Mei.
\newblock Line: Large-scale information network embedding.
\newblock In {\em Proceedings of the 24th International Conference on World
  Wide Web}, pages 1067--1077. International World Wide Web Conferences
  Steering Committee, 2015.

\bibitem{tenenbaum2000global}
J.~B. Tenenbaum, V.~De~Silva, and J.~C. Langford.
\newblock A global geometric framework for nonlinear dimensionality reduction.
\newblock {\em Science}, 290(5500):2319--2323, 2000.

\bibitem{tian2014learning}
F.~Tian, B.~Gao, Q.~Cui, E.~Chen, and T.-Y. Liu.
\newblock Learning deep representations for graph clustering.
\newblock In {\em Proceedings of the Twenty-Eighth AAAI Conference on
  Artificial Intelligence}, pages 1293--1299, 2014.

\bibitem{tolstikhin2017wasserstein}
I.~Tolstikhin, O.~Bousquet, S.~Gelly, and B.~Schoelkopf.
\newblock Wasserstein auto-encoders.
\newblock {\em arXiv preprint arXiv:1711.01558}, 2017.

\bibitem{tu2018structural}
K.~Tu, P.~Cui, X.~Wang, F.~Wang, and W.~Zhu.
\newblock Structural deep embedding for hyper-networks.
\newblock In {\em Thirty-Second AAAI Conference on Artificial Intelligence},
  pages 426--433, 2018.

\bibitem{tu2018deep}
K.~Tu, P.~Cui, X.~Wang, P.~S. Yu, and W.~Zhu.
\newblock Deep recursive network embedding with regular equivalence.
\newblock In {\em Proceedings of the 24th ACM SIGKDD International Conference
  on Knowledge Discovery \& Data Mining}, pages 2357--2366. ACM, 2018.

\bibitem{vilnis2014word}
L.~Vilnis and A.~McCallum.
\newblock Word representations via gaussian embedding.
\newblock {\em arXiv preprint arXiv:1412.6623}, 2014.

\bibitem{vishwanathan2010graph}
S.~V.~N. Vishwanathan, N.~N. Schraudolph, R.~Kondor, and K.~M. Borgwardt.
\newblock Graph kernels.
\newblock {\em The Journal of Machine Learning Research}, 11:1201--1242, 2010.

\bibitem{wang2016structural}
D.~Wang, P.~Cui, and W.~Zhu.
\newblock Structural deep network embedding.
\newblock In {\em Proceedings of the 22nd ACM SIGKDD international conference
  on Knowledge discovery and data mining}, pages 1225--1234. ACM, 2016.

\bibitem{wang2014bregman}
H.~Wang and A.~Banerjee.
\newblock Bregman alternating direction method of multipliers.
\newblock In {\em Advances in Neural Information Processing Systems}, pages
  2816--2824, 2014.

\bibitem{werbos1990backpropagation}
P.~J. Werbos.
\newblock Backpropagation through time: what it does and how to do it.
\newblock {\em Proceedings of the IEEE}, 78(10):1550--1560, 1990.

\bibitem{zang2017long}
C.~Zang, P.~Cui, C.~Faloutsos, and W.~Zhu.
\newblock Long short memory process: Modeling growth dynamics of microscopic
  social connectivity.
\newblock In {\em Proceedings of the 23rd ACM SIGKDD International Conference
  on Knowledge Discovery and Data Mining}, pages 565--574. ACM, 2017.

\bibitem{zhu2018deep}
D.~Zhu, P.~Cui, D.~Wang, and W.~Zhu.
\newblock Deep variational network embedding in wasserstein space.
\newblock In {\em Proceedings of the 24th ACM SIGKDD International Conference
  on Knowledge Discovery \& Data Mining}, pages 2827--2836. ACM, 2018.

\bibitem{Zhu02learningfrom}
X.~Zhu and Z.~Ghahramani.
\newblock Learning from labeled and unlabeled data with label propagation.
\newblock Technical report, 2002.

\bibitem{Zhu2003-SLU}
X.~Zhu, Z.~Ghahramani, and J.~Lafferty.
\newblock Semi-supervised learning using gaussian fields and harmonic
  functions.
\newblock ICML'03, pages 912--919. AAAI Press, 2003.

\bibitem{zhuang2011two}
J.~Zhuang, I.~W. Tsang, and S.~Hoi.
\newblock Two-layer multiple kernel learning.
\newblock In {\em International conference on artificial intelligence and
  statistics}, pages 909--917, 2011.

\end{thebibliography}

\end{document}